\documentclass[accepted]{uai2023} %

\makeatletter
\def\blfootnote{\gdef\@thefnmark{}\@footnotetext}
\makeatother

\usepackage[american]{babel}
\usepackage{svg, amsfonts}

\usepackage{subcaption}
\usepackage[titlenumbered,ruled,resetcount,noend]{algorithm2e}
\usepackage{algpseudocode}
\algdef{SE}[SUBALG]{Indent}{EndIndent}{}{\algorithmicend\ }%
\algtext*{Indent}
\algtext*{EndIndent}

\usepackage{natbib} %
\usepackage{mathtools} %

\usepackage{ebproof}
\usepackage{amsmath}
\usepackage{amsthm}
\usepackage{amssymb}
\usepackage{stmaryrd}

\usepackage{booktabs} %
\usepackage{tikz} %
\usepackage{wrapfig}
\usepackage{float}
\usepackage{caption}
\usepackage{subcaption}
\usepackage{graphicx}
\usepackage{listings}
\usetikzlibrary{patterns,calc,backgrounds}

\newtheorem{theorem}{Theorem}
\newtheorem{lemma}[theorem]{Lemma}
\newtheorem{proposition}[theorem]{Proposition}

\theoremstyle{definition}
\newtheorem{definition}{Definition}

\tikzstyle{nnf}=[
  >=stealth,font=\small,auto,scale=0.7,every node/.style={scale=0.7}
]
\tikzstyle{extnode}=[
  draw,circle,inner sep=2pt,fill=white
]

\tikzstyle{leafnode}=[
  draw,fill=gray!20,inner sep=3.5pt
]
\tikzstyle{constnode}=[
  draw,fill=white,inner sep=3.5pt
]
\tikzstyle{label}=[
  fill=white,inner sep=2.5pt
]

\tikzstyle{acarrow}=[
    decoration={markings,mark=at position 1 with {\arrow[scale=0.6]{>}}},
    postaction={decorate},
    shorten >=0.4pt,
    >=latex,
    line width=0.1
]

\tikzstyle{bnarrow}=[
    decoration={markings,mark=at position 1 with {\arrow[scale=1.5]{>}}},
    postaction={decorate},
    shorten >=0.7pt,
    >=latex,
    line width=0.3
]
\tikzstyle{bayesnet}=[
  >=latex, thick, auto
]
\tikzstyle{bnnode}=[
  draw,ellipse,minimum size=7mm,inner sep=1pt,font=\small
]
\tikzstyle{cpt}=[
  font=\footnotesize
]

\tikzstyle{graph}=[
  >=stealth,font=\small,auto,scale=1,every node/.style={scale=1}
]
\tikzstyle{node}=[
  draw,circle,inner sep=3pt,fill=white
]

\tikzstyle{bdd}=[
  >=latex, thick, >=stealth, font=\small,auto,scale=0.9,every node/.style={scale=0.9}
]
\tikzstyle{bddnode}=[
  draw,circle,inner sep=0pt,fill=white,minimum size=5.5mm
]

\tikzstyle{highedge}=[
    line width=0.9
]
\tikzstyle{lowedge}=[
    line width=0.9,dotted
]
\tikzstyle{bddterminal}=[
  draw,fill=gray!20,inner sep=2.5pt, font=\small
]

\tikzstyle{bddroot}=[
  draw,fill=white,inner sep=2.5pt, font=\small
]

\tikzstyle{bdd}=[
  >=latex, thick, >=stealth, font=\small,auto,scale=0.9,every node/.style={scale=0.9,line width=0.5}
]
\tikzstyle{bddnode}=[
  draw,circle,inner sep=1.5pt,fill=white,minimum size=5.5mm,font=\small
]
\tikzstyle{root}=[
    draw, rectangle, thin, inner sep=2pt, minimum size=5.5mm,font=\small
]

\tikzstyle{edge}=[
    ->, filled
]
\tikzstyle{highedge}=[
    line width=0.9
]
\tikzstyle{lowedge}=[
    line width=0.9,dotted
]
\tikzstyle{bddterminal}=[
  draw,fill=gray!20,inner sep=3.5pt, font=\small
]

\lstdefinestyle{compact}{
  \ttfamily\tiny
}

\usepackage{forest}
\usepackage{tikz}
\usetikzlibrary{shapes,arrows}
\usetikzlibrary{arrows.meta}
\usetikzlibrary{positioning}
\tikzset{
  low/.style={densely dashed, high},
  high/.style={-{Stealth[]}},
}
\forestset{
  BDT/.style={
    for tree={
      {rectangle, minimum size=1.9em, inner sep = 2em},
      if n children=0{rectangle, minimum size=1.9 em}{circle, inner sep=0.2em},%
      draw,%
      edge={
        high,%
      },
    }
  },
}

\newcommand{\true}[0]{ \mathtt{T} }
\newcommand{\false}[0]{ \mathtt{F} }
\definecolor{amethyst}{rgb}{0.6, 0.4, 0.8}

\newcommand{\op}[2]{\operatorname{#1}\mathopen{}\left(#2\right)\mathclose{}}
\newcommand{\opsub}[3]{\operatorname{#1}_{#2}\mathopen{}\left(#3\right)\mathclose{}}
\newcommand{\low}[1]{\op{low}{#1}}
\newcommand{\high}[1]{\op{high}{#1}}
\newcommand{\lsb}[2]{\opsub{LSB}{#1}{#2}}

\newcommand{\ftobdd}[1]{\left\llbracket#1\right\rrbracket}
\newcommand{\bddrt}[2]{{#1}_{#2}(\vec{x})}

\newcommand{\itesty}[1] {#1}
\newcommand{\ite}[3]{ \text{\itesty{if} \quad } #1 \text{\quad \itesty{then} \quad} #2 \text{\quad \itesty{else} \quad} #3 }
\newcommand{\itep}[3]{\left(\ite{#1}{#2}{#3}\right)}

\newcommand{\mhl}[1]{#1}

\newcommand{\BWH}{BITWISE\_INT}

\newcommand{\concat}{^\frown}

\newtheorem{innerpropositionhalf}{Proposition}
\newenvironment{propositionhalf}[1]
  {\renewcommand\theinnerpropositionhalf{#1}\innerpropositionhalf}
  {\endinnerpropositionhalf}

\newcommand{\sh}[1]{}

\DeclareMathOperator\flip{flip}
\DeclareMathOperator\Beta{Beta}
\DeclareMathOperator\uniform{uniform}
\usepackage{pifont}
\newcommand{\xmark}{\ding{55}}
\usepackage{pgfplots}
\usetikzlibrary{plotmarks}
\pgfplotstableread[col sep = comma]{data/lessthan.csv}\lessthandata
\pgfplotstableread[col sep = comma]{data/equals.csv}\equalsdata
\pgfplotstableread[col sep = comma]{data/plus.csv}\plusdata
\pgfplotstableread[col sep = comma]{data/luhn.csv}\luhndata

\title{Scaling Integer Arithmetic in Probabilistic Programs}

\author[1]{William X. Cao}
\author[1]{Poorva Garg$^*$}
\author[2]{Ryan Tjoa$^*$}
\author[3]{Steven Holtzen}
\author[1]{Todd Millstein}
\author[1]{Guy Van den Broeck}
\affil[1]{%
    Department of Computer Science\\
    University of California\\
    Los Angeles, California, USA
}
\affil[2]{
    Department of Computer Science\\
    University of Washington\\
    Seattle, Washington, USA
}
\affil[3]{%
    Khoury College of Computer Sciences\\
    Northeastern University\\
    Boston, Massachusetts, USA
}

\begin{document}
\maketitle

\blfootnote{$^*$These authors contributed equally to this work.}

\begin{abstract}
Distributions on integers are ubiquitous in probabilistic modeling but remain challenging for many of today's probabilistic programming languages (PPLs). The core challenge comes from discrete structure: many of today's PPL inference strategies rely on enumeration, sampling, or differentiation in order to scale, which fail for high-dimensional complex discrete distributions involving integers. 
Our insight is that there is structure in arithmetic that these approaches are not using. We present a binary encoding strategy for discrete distributions that exploits the rich logical structure of integer operations like summation and comparison. We leverage this structured encoding with knowledge compilation to perform exact probabilistic inference, and show that this approach scales to much larger integer distributions with~arithmetic. 
\end{abstract}

\begin{figure*}
\centering
\begin{minipage}[b]{.59\textwidth}
\begin{lstlisting}
id = [discrete([0.72, 0.01, 0.01, 0.01, 0.01, 
                     0.01, 0.2, 0.01, 0.01, 0.01]),...,
        discrete([0.01, 0.01, 0.05, 0.01, 0.01,
                     0.63, 0.2, 0.01, 0.01, 0.05])]
check_digit = id[0]
remaining_id = id[1:] //tail of the array
check_val = luhn_checksum(remaining_id)
observe((check_digit + (check_val %
return id
\end{lstlisting}
\caption{A probabilistic program for the student ID probabilistic inference problem using integer random variables (discrete), integer arithmetic (the Luhn algorithm function), and Bayesian conditioning (observe)}\label{example2} 
\end{minipage}
\quad
\begin{minipage}[b]{.37\textwidth}
\begin{lstlisting}[language=Python]
def luhn_checksum(id)
  sum = 0
  for i in 0..length(id) - 1
    if i %
      if id[i] > 4:
        sum += 2 * id[i] - 9
      else:
        sum += 2 * id[i] 
    else: 
        sum += id[i]
  return sum 
\end{lstlisting}
\caption{Luhn algorithm implementation}\label{example1}
\end{minipage}
\end{figure*}

\section{Introduction}\label{sec:intro}
Probabilistic programming languages (PPLs) are expressive languages for defining
probability distributions. The core idea of a PPL is to enrich a programming 
language with the ability to define, observe, and compute with random variables: hence, the program itself defines a probabilistic model.
This paper focuses on a particular programming feature: scaling inference for programs with random integers and integer arithmetic. Integers are very challenging for today's approaches to probabilistic inference.
The relationships between integer-valued random variables can be very complex: they can be added,
multiplied, compared, etc. This rich structure is opaque to today's inference strategies. Trace-based sampling like Markov-Chain Monte Carlo, importance sampling, and sequential Monte-Carlo all collapse integer distributions
to a single sampled point \citep{gelman2015stan, bingham2019pyro, dillon2017tensorflow,van2018introduction,lew2019trace}. These approximate inference strategies can scale well in many cases, but they struggle to find valid sampling regions in the presence of low-probability observations and non-differentiability (e.g., 
observing the sum of two large random integers to be a constant)~\citep{gelman2015stan, bingham2019pyro, dillon2017tensorflow}. 
Exact inference strategies work by preserving the global structure of the distribution, but here there is a challenge: \emph{what is the right strategy for efficiently representing and manipulating distributions on integers}?
Today's PPLs that support exact inference and integer manipulation -- such as {Dice}~\citep{HoltzenOOPSLA20}, {ProbLog}~\citep{de2007problog}, Psi~\citep{gehr2016psi}, and WebPPL~\citep{dippl} -- model integer distributions using what is essentially a one-hot categorical encoding (i.e., an integer distribution $[\texttt{0} \mapsto 0.25, \texttt{1} \mapsto 0.25, \texttt{2} \mapsto 0.25, \texttt{3} \mapsto 0.25]$ is represented simply as a vector). This encoding style is not capable of exploiting the structure of addition: adding two random variables effectively requires full enumeration.

Our first contribution is a new representation of distributions on integers as distributions on \emph{binary encodings}.
For instance, in the above example, rather than representing the distribution as an exhaustive map from integer values to probabilities, we represent it as a joint distribution on binary bits
$
    [\texttt{00} \mapsto 0.25, \texttt{01} \mapsto 0.25, \texttt{10} \mapsto 0.25, \texttt{11} \mapsto 0.25].
$
The upsides of this seemingly-equivalent representation are twofold. First, we can more efficiently represent the joint distribution itself when it has certain structure. In this case, because the distribution is uniform, we can represent it as a product of two independent Bernoulli distributions, one for each bit: we will show that the ability to factorize the distribution in this manner leads to significant performance improvements. Second, this binary representation reveals the structure of arithmetic: for instance, we can compare two integers by \emph{independently} comparing each of their binary digits and aggregating the results. 

Clearly a binary representation of integers reveals structure, but how can we automatically find and exploit this structure during inference in a PPL?
As our second contribution, we show that two of today's PPLs -- Dice~\citep{HoltzenOOPSLA20} and ProbLog \citep{de2007problog,Fierens2015} -- are already capable of exploiting this structure if it is properly encoded into the program, by virtue of their \emph{knowledge compilation} approach to inference.
We give a lightweight strategy for encoding integer distributions, and show empirically that when using our new binary-encoded distributions these two languages scale to significantly larger and more complex integer distributions without essential modifications to their existing inference~strategies.

As our third contribution we show that scalable support for random integer arithmetic allows us to push the boundaries of discrete probabilistic programming systems in surprising ways. We demonstrate how to model a Beta distribution, a continuous distribution, using probabilistic integers. This modelling method exploits the conjugacy property of the Beta distribution, through which we can always characterize the distribution through its (integral) sufficient statistics. By doing so, we can use the Beta distribution as a prior for Bayesian learning.

The structure of this paper is as follows: Section \ref{sec:motivation} gives a motivating example for integer arithmetic. Section \ref{sec:math} explains our integer representation and explores how common integer operations on this representation have structure exploitable by knowledge compilation. Section \ref{section:experiments} empirically evaluates our representation strategy against existing PPLs. Section \ref{sec:betabern} explores the representation of a continuous Beta prior with random integers. Sections \ref{sec:related} and \ref{sec:conclusion} discuss related work and conclude respectively.  

\section{Motivation}\label{sec:motivation}
We begin with a motivating example highlighting how integer distributions are used in probabilistic programs.
Consider the following probabilistic model based on student ID numbers. Suppose that an optical character recognition system is attempting to parse a handwritten student ID number. For each digit of the ID, it produces a probability distribution representing its beliefs about what the digit could be. Combining this output, we get a probability distribution over all possible student IDs. 

The Luhn algorithm~\citep{luhn} is a commonly used method of validating various ID numbers including student IDs. Given a starting ID such as 70733428, the algorithm provides for a way to compute a sum over the ID, giving us a check digit (4) which is then prepended to the original ID to get a final ID: 470733428. This ID is the one actually issued to a student; when provided with an ID, we can validate it by recomputing the sum and looking at the check digit. 

We wish to use the fact that the student ID can be validated to additionally inform our single-digit distributions from the OCR system. We can implement this as a probabilistic program like the one in Figure \ref{example1}.  Figure \ref{example1} implements a function \verb#luhn_checksum# that takes as input a list representing the digits of the student ID, excluding the check digit. It then does computation according to the Luhn algorithm to compute a sum over the digits, which is then returned. Figure \ref{example2} then uses this function: we create a list \verb#id# which contains distributions over integers derived from the OCR system. The syntax \verb#discrete(v)# for a vector $v = [p_0, .., p_n]$ creates a distribution over the numbers $0, .., n$, in which the number $i$ has the probability $p_i$, and is used in our program to represent said OCR distributions. We call the \verb#luhn_checksum# function on these integer distributions to get a distribution over checksums and condition using the \verb#observe# keyword in line 9 to get an updated distribution over IDs.

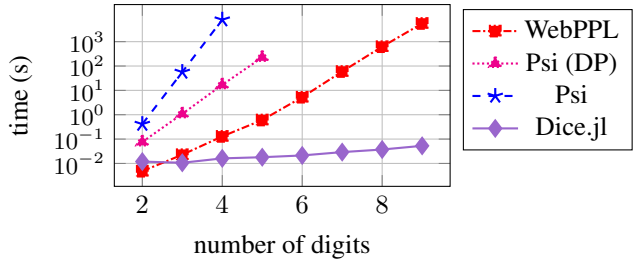
\begin{figure}[t]
    \centering
    \begin{tikzpicture}
    \begin{axis}[
        height=4cm,
        width=6cm,
        grid=major,
        xmode=linear,
        ymode=log,
        xtick={2,4,6,8,10},
        ytick={0.001, 0.01, 0.1, 1, 10, 100, 1000},
        yticklabel style = {font=\footnotesize},
        xlabel={number of digits},
        ylabel={time (s)},
        legend columns=1,
        legend style={at={(12em,4em)},anchor=north west},
        yminorticks=false,
    ]
    \addplot[mark=square*, thick, red, densely dashdotted, mark options={scale=1}] table [x index={0}, y index={1}] {\luhndata}; %
    \addlegendentry{WebPPL};
    \addplot[mark=triangle*, thick, magenta, densely dotted, mark options={scale=1.3}] table [x index={0}, y index={2}] {\luhndata}; %
    \addlegendentry{Psi (DP)};
    \addplot[mark=star, thick, blue, dashed, 
        mark options={scale=2}] table [x index={0}, y index={3}] {\luhndata}; %
    \addlegendentry{Psi};
    \addplot[mark=diamond*, thick, amethyst, 
        mark options={scale=1.5}] table [x index={0}, y index={4}] {\luhndata}; %
    \addlegendentry{{Dice.jl}};
    \end{axis}
  \end{tikzpicture}
    \caption{Single-marginal performance for ID example on increasing ID lengths. WebPPL and Psi scale exponentially due to having to enumerate all paths.}
    \label{luhn_scaling_plot}
\end{figure}

If we implement this program in today's probabilistic programming languages, we will run into a problem. Even if we only wish to compute the marginal probability over a single digit of the ID, Figure \ref{luhn_scaling_plot} shows that the runtime will scale exponentially in the number of digits in the student ID. Each additional digit will contribute a multiplicative amount to the number of total possible ID instantiations, meaning that any approach involving enumeration is inherently exponential. In practice, this means that programs containing student IDs of a realistic length (9-10 digits) will not run. 

The fact that such straightforward programs fail to scale on existing probabilistic programming systems is the primary motivation behind our work.  We have implemented our encoding of integer distributions in \verb#Dice.jl#, a discrete PPL embedded in Julia that uses the
same knowledge compilation approach as Dice~\citep{HoltzenOOPSLA20}. Figure \ref{luhn_scaling_plot} shows that our technique allows inference for such programs to scale for a larger, more realistic number of digits.

\section{Representing \& Manipulating Integer Distributions}\label{sec:math}

This section describes the key technical details behind a binary encoding approach and explains how such an encoding allows the knowledge compilation inference strategy used by Dice and ProbLog to automatically exploit arithmetic structure. We first provide a brief introduction to inference via knowledge compilation. We then demonstrate and analyze various approaches to constructing distributions over integers within probabilistic programs. Finally, we show how the binary encoding can be leveraged by knowledge compilation to identify and exploit contextual independencies for inference over distributions with integer arithmetic.

\subsection{Integer Distributions via BDDs}

\begin{figure}
\centering
\begin{subfigure}[b]{\linewidth}
     \centering
    \begin{tikzpicture}
      \def\lvl{20pt}

    \node (f1) at (0, 0) [bddnode] {$f_1$};
    \node (f2) at ($(f1) + (-7bp, -\lvl)$) [bddnode] {$f_2$};
    \node (f3) at ($(f2) + (-7bp, -\lvl)$) [bddnode] {$f_3$};
    \node (f4) at ($(f3) + (-0bp, -\lvl)$) [bddnode] {$f_4$};

    \node[draw, rounded corners] at ($(f4) + (-30bp, 60bp)$)  {.1};
    \node[draw, rounded corners] at ($(f4) + (-30bp, 40bp)$)  {.11};
    \node[draw, rounded corners] at ($(f4) + (-30bp, 20bp)$)  {.25};
    \node[draw, rounded corners] at ($(f4) + (-30bp, 0bp)$)  {.5};

    \node (false) at ($(f4) + (30bp, -1*\lvl)$) [bddterminal] {$\false$};

    \node (true) at ($(f4) + (90bp, -1*\lvl)$) [bddterminal] {$\true$};

    \node (f6) at ($(f1) + (20bp, 0bp)$) [bddnode] {$f_1$};
    \node (f7) at ($(f6) + (10bp, -\lvl)$) [bddnode] {$f_2$};
    \node (f8) at ($(f7) + (20bp, -\lvl)$) [bddnode] {$f_3$};
    \node (f9) at ($(f8) + (45bp, -\lvl)$) [bddnode] {$f_4$};

    \node (f10) at ($(f6) + (35bp, 0bp)$) [bddnode] {$f_{1}$};
    \node (f11) at ($(f10) + (40bp, -\lvl)$) [bddnode] {$f_{2}$};
    \node (f12) at ($(f11) + (-20bp, -\lvl)$) [bddnode] {$f_{3}$};

    \node (b2) at ($(f1) + (0bp, 20bp)$) [bddroot] {$b_2$};
    \node (b1) at ($(f6) + (0bp, 20bp)$) [bddroot] {$b_1$};
    \node (b0) at ($(f10) + (0bp, 20bp)$) [bddroot] {$b_0$};

    \begin{scope}[on background layer]
      \draw [highedge] (f1) -- (false);
      \draw [lowedge] (f1) -- (f2);
      \draw [highedge] (f2) -- (false);
      \draw [lowedge] (f2) -- (f3);
      \draw [highedge] (f3) -- (false);
      \draw [lowedge] (f3) -- (f4);

      \draw [highedge] (f4) -- (false);
      \draw [lowedge] (f4) -- (true);
      \draw [highedge] (f6) -- (false);
      \draw [lowedge] (f6) -- (f7);
      \draw [highedge] (f7) -- (false);
      \draw [lowedge] (f7) -- (f8);
      \draw [highedge] (f8) -- (true);
      \draw [lowedge] (f8) -- (f9);
      \draw [highedge] (f9) -- (true);
      \draw [lowedge] (f9) -- (false);

      \draw [highedge] (f10) -- (false);
      \draw [lowedge] (f10) -- (f11);
      \draw [highedge] (f11) -- (true);
      \draw [lowedge] (f11) -- (f12);
      \draw [highedge] (f12) -- (false);
      \draw [lowedge] (f12) -- (f9);

      \draw[-stealth] ($(f1) + (0bp, 20bp)$) -- (f1);
      \draw[-stealth] ($(f6) + (0bp, 20bp)$) -- (f6);
      \draw[-stealth] ($(f10) + (0bp, 20bp)$) -- (f10);

    \end{scope}
    \end{tikzpicture}
    \caption{
      Multi-rooted BDD for a CATEG\_INT encoded integer.
    }
    \label{fig:sbk_bdd}
  \end{subfigure}
\vfill

\begin{subfigure}[b]{\linewidth}
     \centering
    \begin{tikzpicture}
      \def\lvl{20pt}

    \node (f1) at (0, 0) [bddnode] {$f_1$};
    \node (f2) at ($(f1) + (40bp, 0bp)$) [bddnode] {$f_1$};
    \node (f3) at ($(f2) + (-15bp, -20bp)$) [bddnode] {$f_2$};
    \node (f4) at ($(f2) + (20bp, 0bp)$) [bddnode] {$f_1$};
    \node (f5) at ($(f4) + (20bp, -20bp)$) [bddnode] {$f_2$};
    \node (f6) at ($(f5) + (-5bp, -40bp)$) [bddnode] {$f_4$};
    \node (f7) at ($(f5) + (-15bp, -20bp)$) [bddnode] {$f_3$};

    \node (false) at ($(f7) + (-20bp, -40bp)$) [bddterminal] {$\false$};
    \node (true) at ($(false) + (-45bp, 00bp)$) [bddterminal] {$\true$};

    \node (b1) at ($(f1) + (0bp, 20bp)$) [bddroot] {$b_2$};
    \node (b2) at ($(f2) + (0bp, 20bp)$) [bddroot] {$b_1$};
    \node (b3) at ($(f4) + (0bp, 20bp)$) [bddroot] {$b_0$};

    \node[draw, rounded corners] at ($(f1) + (-30bp, 0bp)$)  {.3};
    \node[draw, rounded corners] at ($(f1) + (-30bp, -20bp)$)  {.714};
    \node[draw, rounded corners] at ($(f1) + (-30bp, -40bp)$)  {.5};
    \node[draw, rounded corners] at ($(f1) + (-30bp, -60bp)$)  {.6};

    \begin{scope}[on background layer]
      \draw [highedge] (f1) -- (true);
      \draw [lowedge] (f1) -- (false);
      \draw [highedge] (f2) -- (false);
      \draw [lowedge] (f2) -- (f3);
      \draw [highedge] (f3) -- (true);
      \draw [lowedge] (f3) -- (false);

      \draw [highedge] (f4) -- (false);
      \draw [highedge] (f5) -- (f6);
      \draw [lowedge] (f4) -- (f5);
      \draw [lowedge] (f5) -- (f7);
      \draw [highedge] (f6) -- (true);
      \draw [lowedge] (f6) -- (false);
      \draw [highedge] (f7) -- (true);
      \draw [lowedge] (f7) -- (false);

      \draw[-stealth] (b1) -- (f1);
      \draw[-stealth] (b2) -- (f2);
      \draw[-stealth] (b3) -- (f4);

    \end{scope}

    \end{tikzpicture}
    \caption{
      Multi-rooted BDD for an BITWISE\_INT encoded integer.
    }
    \label{fig:bwh_bdd}
  \end{subfigure}
\caption{BDDs representing the integer distribution $[\texttt{0} \mapsto 0.1, \texttt{1} \mapsto 0.1, \texttt{2} \mapsto 0.2, \texttt{3} \mapsto 0.3, \texttt{4} \mapsto 0.3]$ resulting from CATEG\_INT and BITWISE\_INT encoding methods. BITWISE\_INT achieves a smaller BDD by compactly representing higher order bits.
}\label{bddexample}
\end{figure}
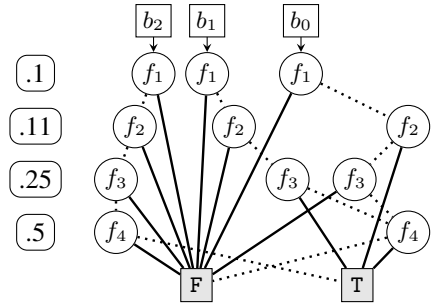
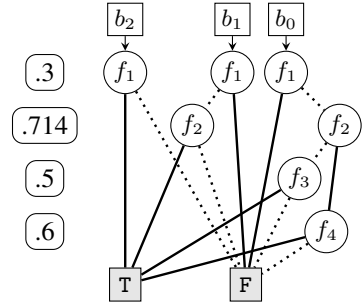

Thus far we have seen how binary-represented distributions expose structure and can enable effective scaling in practice. In this section we explain exactly how this performance improvement is achieved during inference. 
In particular, we show how \emph{knowledge compilation} is capable of automatically finding and exploiting the structure of integer distributions and operations. Inference via knowledge compilation is currently the state-of-the-art approach to exact discrete probabilistic inference in certain classes of probabilistic programs~\citep{HoltzenOOPSLA20, de2007problog, chavira2006compiling, Chavira2008, Fierens2015}. 
The heart of inference via knowledge compilation is a reduction from inference to \emph{weighted model counting} (WMC). Let $\varphi$ be a Boolean formula and $w$ be a map from literals in $\varphi$ to real-valued weights; the pair $(\varphi, w)$ is called a \emph{weighted Boolean formula}. Then, the \emph{weighted model count} $\mathtt{WMC}(\varphi, w)$ is a weighted sum of models of $\varphi$:

\begin{align}
    \mathtt{WMC}(\varphi, w) = \sum_{m \models \varphi} \prod_{\ell \in m} w(\ell).
\end{align}
This reduction to WMC is not useful on its own however: the WMC task is \#P-hard for an arbitrary Boolean formula $\varphi$. This is where knowledge compilation comes into the picture: $\varphi$ is compiled into a data structure that supports efficient weighted model counting in the size of the data structure. A common example of such a knowledge-compilation data structure is a binary decision diagram (BDD), which supports linear-time WMC, but there are many others~\citep{Darwiche_2002}.
PPLs like Dice and ProbLog work by compiling a program into a BDD or related compilation target and thereby reducing probabilistic program inference to WMC on that target~\citep{chavira2005compiling,sang2005performing}. 

The cost of the knowledge compilation approach to inference is almost entirely determined by the structure of the program; the more structure that exists, the more compact the resulting BDD or related data structure can be.
This leads to our core contribution: a new logical representation of integer distributions that 
is amenable to efficient compilation into BDDs. 
To demonstrate how BDDs can encode integer distributions,
Figure~\ref{bddexample} shows two different multi-rooted BDDs that represent \emph{the same distribution} on 
integers.
In both cases the roots in each BDD represent random variables for each binary digit: $b_0$ is the 0-order digit,
$b_1$ is the 1-order digit. 
Weights of each positive literal are shown on the left (with the negative literal weight being 1 minus 
the positive literal weight); dotted edges represent a false assignment and solid edges represent true assignment.
Intuitively, a binary representation has the potential to be more compact than a naive categorical
representation due to the reduction in the number of roots: for instance, Dice~\citep{HoltzenOOPSLA20} 
requires one root for each possible integer value.

As an example of how to use these data structures, consider computing the marginal probability of the high-order bit $b_2$ being true.  This is 
$\mathtt{WMC}(b_2, w)$, which is $0.3$ -- that is clear in Figure~\ref{fig:bwh_bdd} since the sole path from $b_2$ to the true node has weight 0.3, but it is also true for the sole path from $b_2$ to the true node in Figure~\ref{fig:sbk_bdd}, which has weight $(0.9 * 0.89 * 0.75 * 0.5) \approx 0.3$.
In general, to compute the probability of an arbitrary integer, we convert it into binary and 
conjoin the appropriate roots: for instance, to compute the probability of the integer \texttt{0}, 
we compute $\mathtt{WMC}(\overline{b_0} \land \overline{b_1} \land \overline{b_2}, w)$.

\subsection{Integer Encodings}

The previous section demonstrated the potential for binary encodings in knowledge compilation, but how do we connect this to probabilistic programs? In this section we give lightweight encoding strategies for translating 
\verb#discrete(..)# syntax for an arbitrary distribution over integers into distributions on Booleans, which knowledge-compilation-based languages like Dice and ProbLog are already capable of representing. How might such distributions be represented in a probabilistic programming language in practice? To make the problem concrete, we define the integer representation problem as follows: given an input vector $[p_0, .., p_w]$, we want a method which returns a distribution over integers taking on value $i$ with probability $p_i$. While this is relatively restricted by demanding that our distribution is contiguous with lowest value 0, we can convert this to other distributions (for example) by adding an offset or multiplying by a constant.

\subsubsection{A First Approach}
One natural way of constructing such a categorical distribution, is as a set of if-else statements, with each branch corresponding to a different value. For example, the following probabilistic program snippet would correspond to the integer distribution with probability vector $[0.1, 0.2, 0.3, 0.4]$. The syntax $\flip(\theta)$ used in the program is commonly used in discrete PPLs to represent a Bernoulli random variable with bias $\theta$. 

\begin{lstlisting}[mathescape=true]
if flip(0.1) // Bernoulli(0.1)
    return 0
elseif flip(0.2/0.9)
    return 1
elseif flip(0.3/0.7)
    return 2
else
    return 3 
\end{lstlisting}

We use a sequence of these random flips as arguments to the if-else statements to generate the mixture of numbers; note that we renormalize the flip probability at each step to get the correct distribution. This approach is generalized in Algorithm \ref{alg:categ_int}. This and future algorithms should be interpreted as a general method to represent a distribution over integers in any probabilistic programming language supporting Bernoulli random variables and (non-probabilistic) integers. 
Note that representing a categorical variable in this way is a probabilistic program framing of the SBK encoding presented by \cite{sang2005performing}.

\begin{algorithm}
    \caption{CATEG\_INT ($v \in [0, 1]^{w}$)}
    \label{alg:categ_int}
    \DontPrintSemicolon
    \KwIn{Vector $v$ such that $v[i] \propto $ pr($i$)}
    \color{black}
    \eIf{$w$ == $1$ {\normalfont \textbf{or}} $\flip\left(\frac{v[0]}{\sum v}\right)$}
        {\Return $0$}
        {//\ Recurse on the remainder of $v$\\
        \Return 1 + CATEG\_INT($v$[1:]) 
        } 
        
\end{algorithm}

What would occur if we use Algorithm \ref{alg:categ_int} to represent a distribution over binary-encoded integers in a language such as Dice? The BDD for one distribution represented using this approach is given in Figure \ref{fig:sbk_bdd}. 
Note that for each bit, the decision diagram is essentially a linear chain; intuitively, this corresponds to checking each if-else guard in sequence. 
There is almost no node reuse occurring in this BDD.

\subsubsection{A More Compact Encoding}

We propose an alternative method of representing integers from a probability vector that produces provably more compact BDDs. Rather than constructing our mixture by a linear pass through the probability vector, we can instead divide the vector into two parts, using a divide-and-conquer approach. 
Consider the same example as above, where we are again given as input a probability vector $[0.1, 0.2, 0.3, 0.4]$. To get the wanted distribution, we can conditionally add the value $2$ with probability $\frac{0.3+0.4}{0.1+0.2+0.3+0.4}$, corresponding to the latter half of the vector. Depending on if $2$ is added, we then conditionally add the value $1$, with probability derived from the subvectors $[0.1, 0.2]$ and $[0.3, 0.4]$. An example program implementing this is given below: 

\vspace{1em}

\begin{lstlisting}[mathescape=true]
num = 0
if flip(0.7) // 0.3 + 0.4
    num += 2
    if flip(0.4/0.7)
        num += 1
else
    if flip(0.2/0.3)
        num +=1 
return num 
\end{lstlisting}

This approach is formalized in Algorithm \ref{alg:categbits}. For the sake of simplicity, we assume the input vector is always of length $2^b$ for some number $b$; this means that we always divide the vector into its two halves. For an arbitrary input vector, we can simply pad 0 probability values to fulfill this condition; in practice, this algorithm can easily be adapted to work without this explicit padding.

\begin{algorithm}
    \caption{BITWISE\_INT ($v \in [0, 1]^{2^b}$)}
    \label{alg:categbits}
    \DontPrintSemicolon
    \KwIn{Vector $v$ such that $v[i] \propto $ pr($i$)}
        \color{black}
        $p \gets \frac{\sum_{i = 2^{b-1}}^{2^b - 1} v[i]}{\sum_{i = 0}^{2^b - 1} v[i]}$\;
        \eIf{$\mathit{length}(v)==1$}{\Return $0$}{\eIf{$\flip(p)$}{
        //\ Recurse on second half $\geq 2^{b-1}$
        \\\Return BITWISE\_INT($v[2^{b-1}:2^b]$) + $2^{b-1}$ }
        {
        //\  Recurse on first half $< 2^{b-1}$
        \;\Return BITWISE\_INT($v[0:2^{b-1}] $)}}
\end{algorithm}

Note that while Algorithm \ref{alg:categbits} uses arithmetic to produce the distribution, it only ever adds or return powers of two, which directly correspond to the bits of the integer. Therefore, when implementing the algorithm as a distribution on a tuple of bits, we encode each such addition by simply setting the appropriate bit. For the same example as above, our implementation constructs a tuple of bits $(b_1, b_0)$ corresponding to a binary number such that $b_1 = \flip(0.7)$ and $b_0 =$ if $b_1$ then $\flip(\frac{0.4}{0.7})$ else $\flip(\frac{0.2}{0.3})$. 

How does this method of representing integer distributions differ than the one given before? To see this, we look at the BDD for a distribution written in this manner given in Figure \ref{fig:bwh_bdd}. We can see a clear difference between this BDD and that for the approach given in Algorithm \ref{alg:categ_int}. The most significant bit corresponds to a BDD depending on only one flip, as this corresponds to the largest power of two: only one flip is used to determine its value. For the less significant bits, we add an additional layer of variables for each one, with the number of layers in total corresponding to the number of bits needed to represent the input distribution. This is in contrast to the CATEG\_INT encoding, which requires checking a linear chain of variables for each bit, and so achieves a much more compact BDD representation. 

We formalize this difference in BDD size in Proposition \ref{th:bdd_bound}. Note that variable order can greatly influence the size of a BDD, and finding the optimal variable order is an NP-hard problem~\citep{meinel1998algorithms}; 
 we follow the Dice convention of ordering logical variables using (strict, left-to-right) evaluation order. For example, in Figure ~\ref{example2}, the Boolean variables encoding the \texttt{discrete} distribution on Line 1 occur before the variables in the distribution on Line 2 in the order.

\begin{proposition} \label{th:bdd_bound}
A discrete distribution over the integers $\{0, 1 \ldots, 2^b - 1\}$ compiles to a BDD of size $\Theta(b2^b)$ when represented using CATEG\_INT (Algorithm \ref{alg:categ_int}) and a BDD of size $\Theta(2^b)$ when represented using BITWISE\_INT (Algorithm \ref{alg:categbits}), with variables in flip evaluation order.
\end{proposition}

It is BITWISE\_INT that we have implemented in \texttt{Dice.jl} and experimentally evaluate in the next section.

\subsubsection{Uniform Integers} \label{subs:uniform}

The previous encoding strategy works for arbitrary distributions on integers, but in practice one often encounters common highly-structured distributions such as the uniform. One advantage of our approach is that we can exploit the structure of such distributions in order to scale significantly better than the general approach presented in Algorithm~\ref{alg:categbits}. In particular, the structure of the uniform distribution allows for a special encoding with fully independent flips. 

Since the probability of every integer is equal, we can encode a uniform distribution over integers $\{0, 1, \ldots 2^b-1\}$ by adding the values $2^0, 2^1, \ldots, 2^{b-1}$ independently with probability $0.5$. From a bitwise perspective, this is same as independently setting each bit of the number to be true with probability $0.5$. As an example, consider the uniform distribution over integers $\{0, 1 \ldots 15\}$.  Clearly, each possible instantiation of $(\mathit{flip}(0.5), \mathit{flip}(0.5), \mathit{flip}(0.5), \mathit{flip}(0.5))$ is equally likely, and thus equivalently the probability of each integer in the range. 

\begin{algorithm}
\SetAlgoLined
    \caption{UNIFORM(n)}
    \label{alg:uniformgen}
    \DontPrintSemicolon
    \KwIn{Positive integer $n$}
        $b \gets \lfloor\log_2(n)\rfloor$\;
        \eIf{ $\flip\left(\frac{2^b}{n}\right)$}{
            $\mathit{sum} \gets 0$\;
            \For{$i \gets 0$ \KwTo $b-1$}{
                \uIf{$\flip\left(\frac{1}{2}\right)$} {$\mathit{sum} \gets \mathit{sum} + 2^i$}}
                
        \Return $\mathit{sum}$\;
        }
        {\Return UNIFORM($n - 2^b$) + $2^b$} 
\end{algorithm}

The method described above works for uniform distributions whose range is $2^n$ for some $n$; for ranges that are not a power of 2, we use the fact we can decompose any natural number into a sum of powers of 2. This enables a uniform distribution over any range to be represented as a mixture of multiple uniform distributions over smaller power-of-two ranges. We 
formalize this idea in Algorithm \ref{alg:uniformgen}, which gives a method for representing uniform distributions starting at 0; the correctness of this approach is shown in the appendix. We can then use this approach to achieve any uniform distribution by adding an offset. Just like the previous algorithms, this algorithm is implemented by constructing sequences of bits in a manner equivalent to arithmetic. 

We note that unlike the BITWISE\_INT algorithm, where less significant bits have a dependence on more significant bits, our uniform algorithm leverages independence between the bits. Therefore, the BDD obtained when using UNIFORM is more compact than for our other algorithms, and fewer variables are needed to represent such a distribution.

\subsection{Efficient Integer Operations}

While the binary representations of discrete and uniform distributions over integers are interesting, they do not by themselves necessarily give much advantage. If adding two such distributions still requires an explicit enumeration of all possible sums, then we have not gained much over the existing inference approaches. However, the binary encoding enables us to leverage the structure of integers to do much better than this for common operations. In this section, we demonstrate this for integer comparisons and addition.
\vspace{-0.5em}
\subsubsection{Integer Comparisons}

The comparison operator on binary tuples can be implemented using logic circuits like those in computer hardware. Suppose we compute $a < b$ for two binary numbers $a = 001$ and $b = 100$. The circuit first compares the most significant bits (MSBs) of these numbers, which are 0 and 1 respectively - enough to know that $a < b$ is true. If the two numbers instead had the same MSB, we would need to start this comparison over on remaining bits. This process of computing $a < b$ highlights its key contextual independencies. First, given the MSBs of the operands are different, the result of $a < b$ does not depend on the remaining bits. Second, given the MSBs of the operands are same, the computation on the remaining bits does not depend on the value of the MSBs. This structure gets automatically exploited when we use this standard logic circuit to compare integer distributions, where the inputs are now weighted Boolean formulas represented as BDDs, rather than bits.

More concretely, consider the following probabilistic program which defines two random variables having a uniform distribution over the integers \{0, 1, \ldots, 7\} and then outputs the probability of one integer being less than the other.

\label{equality_prog}
\begin{lstlisting} 
a = uniform(0, 8)
b = uniform(0, 8)
return (a < b)
\end{lstlisting}

Enumerating all the values that $a$ and $b$ can take in the above program would lead to 64 combinations. In contrast, the BDD for the comparison operation has size linear in the number of bits, as it exploits contextual independences.  We later present empirical results demonstrating that this leads directly to better scalability for discrete inference.
\vspace{-0.5em}
\subsubsection{Integer Addition}

Consider two binary numbers $a = 001$ and $b = 100$ that we wish to add. %
The least significant bit (LSB) of $a+b$ is computed as the \emph{xor} of the LSBs of $a$, $b$ and 0 (the initial carry bit). The carry, computed as the \emph{and} of the LSBs of $a$ and $b$, is passed on to the next bit and the same process will be repeated for the remaining bits. The process described above shows that given the carry bit, each bit of the result is independent of the lesser significant bits of the operands.  Similar to the comparison operation, encoding addition on integer distributions as a logical circuit directly exploits such contextual independencies to produce a compact BDD, which in turn leads to significant performance gains. 
The manner in which addition corresponds to a compact BDD has been explored before; \cite{WEGENER2004229} show that for an optimal variable ordering, there is a linear bound on the BDD for addition.

In this section, we described the structure of two arithmetic operations, comparison and addition, independently. When composing these operations together, the compilation of weighted Boolean formulas will naturally compose as described in previous work~\citep{HoltzenOOPSLA20}. 
The sizes of the resulting BDDs depend highly on the variable ordering --- but even when the variable ordering is not optimal, contextual independences can still be identified and exploited, as shown by the experiments in the next section.

\vspace{-0.5em}
\section{Empirical Evaluation}\label{section:experiments}
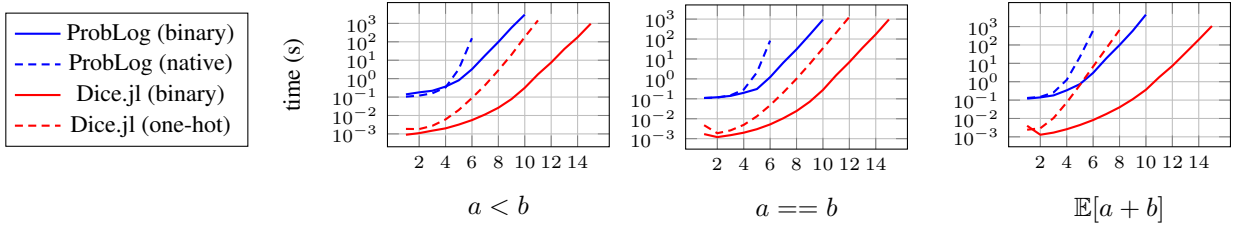
\begin{figure*}[t]
  \centering
\begin{minipage}[t]{0.45\linewidth}
\vspace{0pt}
  \centering
  \begin{tikzpicture}
    \begin{axis}[
        height=3.5cm,
        width=4.5cm,
        grid=major,
        xmode=linear,
        ymode=log,
        xtick={2,4,6,8,10,12,14},
        ytick={0.001, 0.01, 0.1, 1, 10, 100, 1000},
        xlabel={$a<b$},
        ylabel={\small time (s)},
        legend columns=1,
        yticklabel style = {font=\scriptsize},
        xticklabel style = {font=\scriptsize},
        legend style={at={(-15em,2.5em)},anchor=north west, font=\tiny},
        yminorticks=false
    ]
    \addplot[mark=none, thick, blue, mark size = 1.0pt] table [x index={0}, y index={2}] {\lessthandata}; %
    \addlegendentry{\small ProbLog (binary)};
    \addplot[mark=none, thick, blue, densely dashed, mark size = 1.0pt] table [x index={0}, y index={1}] {\lessthandata}; %
    \addlegendentry{\small ProbLog (native)};
    \addplot[mark=none, thick, red, mark size = 1.0pt] table [x index={0}, y index={3}] {\lessthandata}; %
    \addlegendentry{\small {Dice.jl} (binary)};
    \addplot[mark=none, thick, red, densely dashed, mark size = 1.0pt] table [x index={0}, y index={4}] {\lessthandata}; %
    \addlegendentry{\small {Dice.jl} (one-hot)};
    \end{axis}
  \end{tikzpicture}
  \label{fig:lessthan}
\end{minipage}~
\begin{minipage}[t]{0.25\linewidth}
\vspace{0pt}
  \centering
  \begin{tikzpicture}
    \begin{axis}[
		height=3.5cm,
		width=4.5cm,
		grid=major,
    xmode=linear,
    ymode=log,
    xtick={2,4,6,8,10,12,14},
    ytick={0.001, 0.01, 0.1, 1, 10, 100, 1000},
    yticklabel style = {font=\scriptsize},
    xticklabel style = {font=\scriptsize},
    xlabel={$a==b$},
    yminorticks=false
    ]
    \addplot[mark=none, thick, blue] table [x index={0}, y index={2}] {\equalsdata}; %
    \addplot[mark=none, thick, blue, densely dashed] table [x index={0}, y index={1}] {\equalsdata}; %
    \addplot[mark=none, thick, red] table [x index={0}, y index={3}] {\equalsdata}; %
    \addplot[mark=none, thick, red, densely dashed] table [x index={0}, y index={4}] {\equalsdata}; %
    \end{axis}
  \end{tikzpicture}
  \label{fig:laddernetscale}
\end{minipage}~
\begin{minipage}[t]{0.2\linewidth}
\vspace{0pt}
  \centering
      \begin{tikzpicture}
    \begin{axis}[
		height=3.5cm,
		width=4.5cm,
		grid=major,
    xmode=linear,
    ymode=log,
    xtick={2,4,6,8,10,12, 14},
    ytick={0.001, 0.01, 0.1, 1, 10, 100, 1000},
    yticklabel style = {font=\scriptsize},
    xticklabel style = {font=\scriptsize},
    xlabel={$\mathbb{E}[a + b]$},
    yminorticks=false
    ]
    \addplot[mark=none, thick, blue] table [x index={0}, y index={2}] {\plusdata}; %
    \addplot[mark=none, thick, blue, densely dashed] table [x index={0}, y index={1}] {\plusdata}; %
    \addplot[mark=none, thick, red] table [x index={0}, y index={3}] {\plusdata}; %
    \addplot[mark=none, thick, red, densely dashed] table [x index={0}, y index={4}] {\plusdata}; %
  \end{axis}
\end{tikzpicture}
\label{fig:motiv_scale}
\end{minipage}
\caption{Time needed to compute the given operation on two random integers with varying bitwidth (x-axis). 
}
  \label{fig:scaling}
\end{figure*}
In this section, we empirically evaluate our integer compilation strategy. While we have demonstrated that a binary encoding exposes structure that knowledge compilation can exploit, it remains to be seen if this can improve the performance of probabilistic programs. In addition, we have yet to show how our approach compares to other inference methods on larger, more complex arithmetic models. We seek to answer the following questions. 

\begin{enumerate}[label=\textbf{\arabic*}),noitemsep]
    \item Does a binary encoding benefit existing knowledge compilation based languages?
    
    \item Does our approach outperform those of existing PPLs that support exact discrete inference?
\end{enumerate}

To this end, we have implemented our integer representation in \verb#Dice.jl#, a PPL embedded in Julia that uses the same knowledge compilation approach as Dice~\citep{HoltzenOOPSLA20}. 
In \verb#Dice.jl# (binary), unsigned random integers are implemented using the strategy described in the previous section; signed random integers are additionally implemented as a natural extension. The language provides the syntax \verb#discrete(..)# for arbitrary integer distributions and \verb#uniform(..)# for uniform distributions, implemented using the algorithms in Section \ref{sec:math}, as well as the operators $=, <, +, -, *, /,$ and $\%$. Simple operators are implemented as logical circuits, while more complex operators are implemented by composing simpler ones. \footnote{Our implementation and code for all experiments are available at \url{https://github.com/Juice-jl/Dice.jl/tree/arithmetic}.}

Reported runtimes are a median over at least 5 runs; all experiments were run with a 1 hour timeout. 
A best-effort attempt was made for each language to implement benchmarks in a maximally performant manner. More experimental details are available in the appendix.

\subsection{Improving Knowledge Compilation Languages with a Binary Encoding}

In this section, we demonstrate the benefit a binary encoding brings to knowledge compilation based languages. We compare our integer representation method to the methods of the PPL Dice~\citep{HoltzenOOPSLA20} and the probabilistic logic programming language ProbLog~\citep{Fierens2015}, both of which use knowledge compilation as their approach to inference. We do this by comparing the time needed to compute the simple arithmetic operations $a+b$, $a<b$, and $a==b$ on random integers of width $2^n$ for varying $n$ from 1 to 15. Here, the runtime of each method corresponds to the time needed to compile and run inference on each representation, effectively measuring how well the knowledge compilation based inference can exploit the structure of each simple function. For addition, we use the expectation of the sum as our target computation to avoid an output distribution with an exponentially increasing support. 

To allow for a fair comparison between ProbLog's native integer representation and our binary representation, we implement equivalent ProbLog programs computing the arithmetic operations, one using native ProbLog encodings and one using a binary representation. Both programs can then be run using ProbLog, controlling for the specific knowledge compilation system. To compare with Dice's native one-hot integer encoding, we implement both the one-hot encoding and our binary encoding in \verb+Dice.jl+. We then run the simple arithmetic programs using both encodings. %

The results of these experiments are presented in Figure \ref{fig:scaling}. 
We can clearly see that our binary encoding outperforms the existing integer representation strategy used in each language; while at small distribution widths (on the order of $2^4$), they are roughly comparable, our approach scales much better to larger integer distributions.

\subsection{Complex Arithmetic Models}
\begin{table*}[!htp]\centering
\caption{Runtimes in seconds for probabilistic models using integers in various PPLs. \xmark\, indicates a timeout (over 1 hour).}\label{baselines}
\scriptsize
\begin{tabular}{lrrrrr}\toprule
Benchmarks              &{Dice.jl} (binary)         & {Dice.jl} (one-hot)   &WebPPL &Psi (DP) &Psi \\\midrule
book                    &\textbf{5.297}    &\xmark  &\xmark                  &\xmark &\xmark \\\cmidrule{1-6}
tugofwar                &\textbf{0.106}    &2660.373 &21.012             &\xmark &\xmark \\\cmidrule{1-6}
caesar-small            &\textbf{0.041}    &4.968 &0.074 &2.022       &402.196 \\
caesar-medium           &0.239             &39.518 &\textbf{0.135}     &12.505 &\xmark \\
caesar-large            &0.556             &122.109 &\textbf{0.227}     &30.387 &\xmark \\\cmidrule{1-6}
ranking-small           &\textbf{0.007}    &0.025 &0.83 &103.572      &\xmark \\
ranking-medium          &\textbf{0.022}     &0.077&\xmark &318.658         &\xmark \\
ranking-large           &\textbf{0.048}     &0.150&\xmark &330.51          &\xmark \\\cmidrule{1-6}
radar1                  &\textbf{0.034}     &0.664&118.002            &394.525 &2.517 \\\cmidrule{1-6}
floydwarshall-small     &\textbf{0.009}    &0.152 &\textbf{0.009}     &0.115 &113.467 \\
floydwarshall-medium    &\textbf{0.515}    &624.220 &9.51               &2792.14 &\xmark \\
floydwarshall-large     &\textbf{3.406}     &\xmark&\xmark                  &\xmark &\xmark \\\cmidrule{1-6}
linear extensions-small &\textbf{0.003}     &0.004&0.016              &0.351 &5.153 \\
linear extensions-medium &\textbf{0.007}   &0.013 &0.465              &111.38 &\xmark \\
linear extensions-large &\textbf{0.072}    &0.164 &162.009            &\xmark &\xmark \\\cmidrule{1-6}
triangle-small          &\textbf{0.086}     &102.544&3.693              &616.746 &482.14 \\
triangle-medium         &\textbf{0.455}     &1123.171&28.354             &\xmark &\xmark \\
triangle-large          &\textbf{17.365}   &\xmark &\xmark           &\xmark &\xmark \\\cmidrule{1-6}
gcd-small               &2.876             &\xmark &\textbf{0.189}     &24.33 &\xmark \\
gcd-medium              &103.614           &\xmark &\textbf{2.501}     &467.581 &\xmark \\
gcd-large               &\xmark            &\xmark &\textbf{46.626}    &\xmark & \xmark\\\cmidrule{1-6}
disease-small           &7.91              &\xmark &\textbf{1.093}     &109.242 &1009.848 \\
disease-medium          &764.212           &\xmark &\textbf{327.545}   &\xmark &\xmark \\\cmidrule{1-6}
luhn-small              &\textbf{0.039}    &0.594 &0.428              &44.164 &\xmark \\
luhn-medium             &\textbf{4.575}    &23.933 &42.372             &\xmark &\xmark \\
\bottomrule
\end{tabular}
\end{table*}

We also evaluated \verb+Dice.jl+ on more complex models involving distributions over integers. The models were taken from a variety of sources. These include examples involving integers from the existing PPL literature, various examples using continuous distributions adapted to a discrete space, natural modelling tasks using integers such as ranking and text manipulation, and traditional algorithms in a probabilistic setting. A short description of our baselines and their sources is given in the appendix.

As a point of comparison, we use two other PPLs supporting exact discrete inference. We identify two major classes of exact inference approaches used for discrete probabilistic programs: enumerative methods, which work by enumerating all paths through the program, and symbolic methods, which represent and compute the probability distribution through symbolic expressions. We compare against WebPPL \citep{dippl} from the former category and Psi~\citep{gehr2016psi} from the latter. 

We also compare against a version of \verb#Dice.jl# that uses a one-hot encoding of integer distributions as a proxy for existing knowledge compilation approaches; this comparison avoids language-specific differences in performance. Compiled BDD sizes for the programs are provided in the appendix as an additional metric. 

Our results are summarized in Table \ref{baselines}. Many of our benchmark models can naturally be scaled to different sizes; they are implemented in small, medium, and large (corresponding to model size) variants to display the scaling behavior. Psi supports two exact inference algorithms: a default symbolic exact inference algorithm ("Psi") and its specialized dynamic programming inference algorithm ("Psi (DP)"). 

For the majority of benchmarks, our approach outperforms the current exact inference approaches, often achieving an orders-of-magnitude speedup. This result empirically validates the ability of our compilation strategy to exploit arithmetic structure in order to improve inference performance. We observe that the binary encoding outperforms the one-hot encoding with the same underlying knowledge compilation approach, demonstrating its superiority in exposing arithmetic structure. 

We note that \verb#Dice.jl# does not always outperform WebPPL, the enumeration based approach. We make special note of these examples. The caesar example introduces many random integers but immediately makes an observation on their value, thereby reducing the enumeration task and making this tractable for all approaches. The disease example contains parametric distributions on integers; for example, a binomial distribution with parameter $n$ distributed by a uniform distribution. These distributions have much less structure to exploit, and so our approach becomes essentially enumerative, but with additional overhead in compilation. The GCD example, which makes repeated use of the mod ($\%$) operator, similarly has a harder to exploit structure. However, we can see in general that our approach scales well to larger examples and outperforms existing PPLs that support exact discrete inference.

\vspace{-0.5em}

\section{Enabling Continuous Priors with Discrete Distributions}\label{sec:betabern}

Previous sections presented an inference strategy for integer distributions allowing for the scaling of integer arithmetic. We now demonstrate an interesting application of these integers: representing the continuous Beta distribution in a discrete space. 
Continuous priors are an essential part of Bayesian reasoning.
One particularly useful prior for discrete PPLs like \texttt{Dice.jl} is the \emph{Beta prior} $\Beta(\alpha,\beta)$, which is conjugate to the Bernoulli. The $\Beta$ distribution is continuous and thus not amenable to direct representation in \texttt{Dice.jl}.  However, we observe that if a Beta prior for a Bernoulli random variable has integral parameters $\alpha$ and $\beta$, then the posterior distribution is also a Beta with integral parameters. This section explains how we use this observation to represent Beta distribution in \texttt{Dice.jl}.

Assume we have the following program where \texttt{Dice.jl} is extended to permit restricted classes of $\Beta$ priors, where the parameters $\alpha$ and $\beta$ must be constant integers:
\begin{lstlisting}[mathescape=true]
$\theta$ = Beta(1, 2)
x = flip($\theta$)
observe(x)
return $\theta$
\end{lstlisting}
We can perform inference for the posterior by exploiting well-known conjugacy results between $\Beta$ priors and Bernoullis. In particular, observing that \texttt{x} is true, as is done on Line 3 above, increases the \emph{pseudocount} $\alpha$ by 1, making the posterior for $\theta$ become $\Beta(2,2)$. Similarly, observing that \texttt{x} is false increases the pseudocount $\beta$ by 1.

To automate this approach in \texttt{Dice.jl}, we introduce program variables A and B to represent the pseudocounts $\alpha$ and $\beta$, respectively. We then conditionally update these pseudocounts after each \verb+flip+: increment $\alpha$ if the \verb+flip+ returns true; otherwise increment $\beta$.  Doing so ensures that later observations will have the desired effect on the pseudocounts.

The only remaining challenge is that discrete PPLs that employ knowledge compilation, like \verb+Dice.jl+, only support \verb+flip+s whose parameters are constants, so the \verb+flip+$(\theta)$ on Line 2 above is not supported.  To encode it, we use the fact that $\mathbb{E}[\Beta(\alpha, \beta)] = \frac{\alpha}{\alpha+\beta}$, so the \verb+flip+ on line 2 can be simplified to \verb+flip+$(\frac{A}{A+B})$. 
Unfortunately, \texttt{Dice.jl} still does not support this construct, since $\frac{A}{A+B}$ is not a constant.  However, $A+B$ is always a deterministic integer, since each observation increases it by exactly 1.  Therefore, we can introduce a variable T representing A+B and encode \verb+flip+$(\frac{A}{A+B})$ as $\uniform(0, T) < A$ (where $\uniform(0, n)$ is the uniform distribution over the integers ${0,..,n-1}$). Finally, we observe that it is not necessary to maintain both variables A and B, since B is derivable from A and T. The final transformed version of the program is as follows:

\begin{lstlisting}[mathescape=true]
A = 1, T = 3 
x = uniform(0, T) < A 
A = if x then A+1 else A
T = T + 1   
observe(x)
return (A, T-A)
\end{lstlisting}

If we wish to draw another flip on the same Beta prior, we can simply repeat the code in lines 2-4 above. In this way, what we have actually implemented is a Beta-Bernoulli process via a Polya urn model - something analyzed in detail in probabilistic programs by \cite{staton}. We also note that this representation strategy --- an implementation of an urn model --- can also be used for many other distributions in addition to the Beta~\citep{polya}. An example application of this Beta prior --- learning Bayesian network parameters --- is given in the appendix.

\section{Related Work}\label{sec:related}

\emph{PPL Inference.}~~
Knowledge-compilation-based PPLs are most closely-related to this work~\citep{HoltzenOOPSLA20, de2007problog, Fierens2015, saad2021sppl, pfanschilling2022sum}. All these languages stand to benefit from our new binary-encoding. Other PPLs perform exact discrete inference by eliminating discrete variables via enumeration or variable elimination; these approaches lose global structure and hence cannot exploit arithmetic structure as in our approach~\citep{dippl, bingham2019pyro}. Symbolic methods support integers by representing them as a symbolic formula or program~\citep{gehr2016psi,narayanan2016probabilistic}; we believe that in principle it may be possible to adapt these symbolic representations to use a binary representation, but currently these systems do not. Recent work uses probability generating functions (PGFs) to represent (potentially unbounded) discrete distributions~\citep{EPIUGF2023}. 
PGFs represent the distribution symbolically, but do not appear to be compatible with our strategy for binary encodings.
Sampling based inference algorithms work very well for probabilistic programs with continuous distributions, but do not exploit the global structure of integer arithmetic~\citep{KANTAS2009774, hoffman2014no, Arouna+2004+1+24,10.5555/3061053.3061128}.
Finally, there are algorithms that seek to efficiently model integer distributions with recursion~\citep{KY76, pmlr-v108-saad20a}; these approaches are orthogonal to ours, as we do not use recursion.

\emph{Graphical models.}~~
Probabilistic graphical model (PGM) based inference methods~\citep{pfeffer2009figaro, McCallum2009, sanner2012symbolic, koller2009probabilistic} support integers by treating them as categorical distributions. PGMs struggle to represent arithmetic: for instance, a CPT for adding two $n$-bit numbers requires $O(2^n)$ entries.

\section{Conclusion}\label{sec:conclusion}
We presented a strategy for encoding random integers in probabilistic programs via a binary representation, which allows arithmetic operations to be performed through standard Boolean circuits. When combined with the knowledge compilation approach to probabilistic inference, this strategy naturally exploits structure in arithmetic that current approaches do not account for. We showed empirically that this allows existing discrete PPLs to scale to significantly more complex probabilistic models. One interesting consequence is that we can now leverage conjugacy to represent the Beta distribution, a continuous distribution, in purely discrete programs.

\begin{acknowledgements} %
    This work was funded in part by the DARPA PTG Program under contract number HR00112220005, NSF grants \#IIS-1943641, \#IIS-1956441, \#CCF-1837129, \#CCF-2220408, and a gift from RelationalAI. GVdB discloses a financial interest in RelationalAI.

\end{acknowledgements}

\bibliography{cao_757}

\appendix
\onecolumn

\section{Proofs}

\allowdisplaybreaks

\subsection{BDD Sizes of Integer Distribution Encodings (Proposition 1)}

\subsubsection{Preliminaries}

A $b$-rooted BDD $B$ with $m$ decision variables computes some function $\{0, 1\}^m \to \{0, 1\}^b$ at its roots. Note that we can treat each root of the BDD as corresponding to an individual function  $\{0, 1\}^m \to \{0, 1\}$; in this manner we can equivalently treat such a BDD as a tuple of functions $(\{0, 1\}^m \to \{0, 1\})^b$. We will let $B(\vec{x}) = (\bddrt{B}{1}, \dots, \bddrt{B}{b})$ denote this tuple. Here, $\vec{x}$ contains the decision variables of the BDD, with the variable order of the BDD matching the order of the components of $\vec{x}$. 

Note that the function computed by the BDD has as output a bit vector of length $b$. This can be interpreted as a $b$-bit unsigned integer, and we will treat the two interchangeably throughout these proofs. This also suggests that a function $f : \{0, 1\}^m \to \{0, \dots, 2^b - 1\}$ can equivalently specify a BDD; we will use $\ftobdd{f}$ to denote the corresponding BDD. 

We will be dealing with probabilistic BDDs, where each decision variable is given a weight representing a probability. Intuitively, this can be seen as replacing each $x_i$ with a random $X_i = \operatorname{Bernoulli}(p_i)$; each decision variable has positive weight $p_i$ and negative weight $1-p_i$. The BDD can thus be viewed as representing a probability distribution over potential outputs, where the probability corresponds to the weighted paths through the BDD. We will say a probabilistic BDD $B(\Vec{X})$ with probability assignments $\vec{p}$ \textit{encodes} a probability distribution over unsigned integers $\Pr(V)$ if the two distributions are the same, i.e. the distribution over output bits of $(B(\Vec{X}), \vec{p})$, when interpreted as integers, is identical to $\Pr(V)$.

\subsubsection{Categorical Encoding of Integer Distributions}

Recall Algorithm \ref{alg:categ_int}. Expanding the recursion and simplifying the arithmetic structure, we get the following natural encoding method for integers, which is similar to how the algorithm is implemented in practice. 

\begin{gather*}
    g(\vec{x}) = \begin{cases}
        0 & \operatorname{if\ }x_0\\
        1 & \operatorname{else\ if\ }x_1\\
        \dots\\
        2^b - 2 & \operatorname{else\ if\ }x_{2^b - 2}\\
        2^b - 1 & \operatorname{else}
    \end{cases}
\end{gather*}

In Algorithm \ref{alg:categ_int}, each $x_i$ is given some probability $p_i$, corresponding to the $\flip$ in the algorithm. In this algorithm, we recurse on the tail of the input probability vector $v$, and so this probability $p_i = \frac{v[0]}{\sum v}$ represents the conditional probability that we return a number, given that our returned value is greater than or equal to that number. We formalize this encoding method in Definition \ref{categ_int_formal}. 

\begin{definition}
Given a distribution $\Pr(V)$ over the integers $\{0,...,2^b-1\}$,
define $\operatorname{encode_{CATEG}}(\Pr(V)) = (\ftobdd{g}, \vec{p})$, where $p_i = \Pr(V = i | V \geq i)$ for $0 \leq i \leq 2^b-2$, and  $g$ is defined as follows:
\begin{gather*}
    g(\vec{x}) = \begin{cases}
        0 & x_0\\
        1 & \lnot x_0 \land x_1\\
        \dots\\
        j & x_j \land \bigwedge\limits_{i=0}^{j-1}\lnot x_i \operatorname{\ and\ } 0 \leq j \leq 2^b-2\\
        \dots\\
        2^b - 1 & \bigwedge\limits_{i=0}^{2^b-2}\lnot x_i
    \end{cases}
\end{gather*}

Note that with the vector variable order $x_0, x_1, .., x_{2^b-2}$, this exactly matches the program variable order of Algorithm \ref{alg:categ_int}, and thus the evaluation order in the proposition. 
\label{categ_int_formal}
\end{definition}

\begin{lemma} \label{geqinductive}
 $\Pr(v \geq x) = \prod_{i=0}^{x-1} \Pr(v \neq i | v \geq i) \forall x, v \in \mathbb{N}_0$. Proof omitted; proceed by induction.
\end{lemma}
\begin{lemma} \label{sbkencodes}
For any distribution $\Pr(V)$ over the integers $\{0,...,2^b-1\}$, $\operatorname{encode_{CATEG}}(\Pr(V))$ encodes $\Pr(V)$.
\end{lemma}
\begin{proof}
Let $(\ftobdd{g}, \vec{p}) = \operatorname{encode_{CATEG}}(\Pr(V))$. We prove that for all $0 \leq j \leq 2^b - 1$, $\Pr(g(\vec{X})=j) = \Pr(V=j)$.

\textit{Case 1:} $0 \leq j \leq 2^b-2$.
\begin{align*}
    \Pr(g(\vec{X})=j) = \Pr\left(X_{j} \land \bigwedge\limits_{i=0}^{j-1}\lnot X_i\right)
    &=  \Pr(X_{j})\prod_{i=0}^{j-1}\Pr(\lnot X_i) && \text{(Independence)}\\
    &=\Pr(V=j|V\geq j)\prod_{i=0}^{j - 1}\Pr(V \neq i \mid V \geq i) && \text{(As $\Pr(X_i) = p_i$)}\\
    &=\Pr(V=j|V\geq j)\Pr(V \geq j) && \text{(Lemma \ref{geqinductive})}\\
    &=\Pr(V=j)
\end{align*}

\textit{Case 2:} $j = 2^b-1$.
\begin{align*}
    \Pr(g(\vec{X})=2^b-1) = \Pr\left(\bigwedge\limits_{i=0}^{2^b-2}\lnot X_i\right)
    &=\prod_{i=0}^{2^b-2}\Pr(\lnot X_i) && \text{(Independence)}\\
    &=\prod_{i=0}^{2^b-2}\Pr(V \neq i | V \geq i) && \text{(As $\Pr(X_i) = p_i$)}\\
    &= \Pr(V \geq 2^b - 1) && \text{(Lemma \ref{geqinductive})}\\
    &= \Pr(V = 2^b - 1)
\end{align*}
\end{proof}

\subsubsection{BDD Size of the CATEG\_INT Encoding}

We prove the first half of Proposition 1.
\begin{propositionhalf}{1 (first half)} \label{sbksize}
A discrete distribution over the integers $\{0, 1 \ldots, 2^b - 1\}$ compiles to a BDD of size $\Theta(b2^b)$ when represented using CATEG\_INT (Algorithm \ref{alg:categ_int}), with variables in flip evaluation order.
\end{propositionhalf}

We will use $\lsb{i}{j}$ to denote the $i$th least significant bit of $j$; for example, $\lsb{2}{13}$ = 0. 

\begin{lemma} \label{sbkstructure}
Let $(\ftobdd{g}, P) = \operatorname{encode_{CATEG}}(\Pr(V))$, where $\Pr(V)$ is the distribution over the integers $\{0, .., 2^b-1\}$. $\ftobdd{g}$ will be exactly the BDD $C$ with the following structure:

For each root $1 \leq i \leq b$, $C$ has nodes $N_{i, 0}, N_{i, 1}, \dots, N_{i, 2^b - 2^{i-1} - 1}$, at the levels of decision variables $x_0, x_1, \dots, x_{2^b - 2^{i - 1} - 1}$, respectively.
\begin{align*}
    \low{N_{i, j}} &= \text{if \quad } j < 2^b - 2^{i-1} - 1 \text{\quad then \quad} N_{i, j + 1} \text{\quad else \qquad} 1\\
    \high{N_{i, j}} &= \lsb{i}{j}
\end{align*}
The roots of the BDD $C_i$ each point to the corresponding node $N_{i, 0}$.
\end{lemma}

Intuitively, the high edge of a node indicates that that decision variable is true, meaning that we have "chosen" our number (and thus the corresponding value for the bit); the low edge means that we should move on to the next bit. As an example, the BDD for $b = 2$ follows (terminal nodes visually duplicated for clarity). 

\vspace{-10pt}
\begin{center}
    \begin{minipage}[t]{0.15\textwidth}
    \vspace{0pt}
        \begin{center}
        \begin{tikzpicture}
          \def\lvl{35pt}
        \node (c1) at (0, 0) [root] {$C_1$};
        \node (n10) at ($(c1) + (0bp, -30pt)$) [bddnode] {$N_{1,0}$};
        \node (n11) at ($(n10) + (0bp, -\lvl)$) [bddnode] {$N_{1,1}$};
        \node (n12) at ($(n11) + (0bp, -\lvl)$) [bddnode] {$N_{1,2}$};
        \node (true) at ($(n12) + (-15bp, -30pt)$) [bddterminal] {$\true$};
        \node (false) at ($(n12) + (15bp, -30pt)$) [bddterminal] {$\false$};
        
        \begin{scope}[on background layer]
          \draw [-stealth] (c1) -- (n10);
          \draw [lowedge] (n10) -- (n11);
          \draw [lowedge] (n11) -- (n12);
          \draw [lowedge] (n12) -- (true);
          \draw [highedge] (n12) -- (false);
          
          \path[-] (n11) edge [bend right, highedge] node {} (true);
          \path[-] (n10) edge [bend left, highedge] node {} (false);

        \end{scope}
        \end{tikzpicture}
        \end{center}
    \end{minipage}
    \begin{minipage}[t]{0.15\textwidth}
    \vspace{0pt}
    
        \begin{tikzpicture}
          \def\lvl{35pt}
        \node (c2) at (0, 0) [root] {$C_2$};
        \node (n20) at ($(c2) + (0bp, -30pt)$) [bddnode] {$N_{2,0}$};
        \node (n21) at ($(n20) + (0bp, -\lvl)$) [bddnode] {$N_{2,1}$};
        \node (true) at ($(n21) + (-15bp, -30pt-\lvl)$) [bddterminal] {$\true$};
        \node (false) at ($(n21) + (15bp, -30pt-\lvl)$) [bddterminal] {$\false$};
        
        \begin{scope}[on background layer]
          \draw [-stealth] (c2) -- (n20);
          \draw [lowedge] (n20) -- (n21);
          \draw [lowedge] (n21) -- (true);
          \draw [highedge] (n21) -- (false);
          
          \path[-] (n20) edge [bend left, highedge] node {} (false);

        \end{scope}
        \end{tikzpicture}
    \end{minipage}
\end{center}

\begin{proof} 
We argue that the described BDD is reduced. Note that the BDD essentially consists of linear chains from each root node. It is clear that no node has isomorphic children, as one child is always a terminal node while the other child is a decision node, and therefore we cannot eliminate any nodes. As the BDD consists of chains, merging nodes must be done between chains. As described, each chain is of a different depth, and so resolves on a different final decision variable. Therefore, at no stage can two nodes compute the same subfunction, and so we cannot merge nodes. 

As our BDD is reduced, it is canonical for the represented function. Therefore, showing that the function represented by $C$ is the same as $g$ will show that our BDD $\ftobdd{g}$ has the same structure.  

We will now show that for all $\vec{x}$,  $\bddrt{C}{i} = \lsb{i}{g(\vec{x})}$, and thus that the functions are equivalent.

Let $S_{i,j}(\vec{x})$ denote the state of $C_i$ at the level of $x_j$, given a (potentially partial) assignment to $\vec{x}$; this is the subfunction that must be computed at that point. This is usually a node at level $x_j$, but can be a lower level node such as a terminal node (representing that the decision node can be skipped). See the following BDD and corresponding states as an example.

\begin{center}
    \begin{minipage}{0.3\textwidth}
    \begin{center}
        \begin{alignat*}{4}
        S_{1,0}(&x_0,x_1,x_2) &&= N_{1,0}\\
        S_{1,1}(&0, x_1, x_2) &&= N_{1,1}\\
        S_{1,1}(&1, x_1, x_2) &&= 0\\
        S_{1,2}(&0, 0, x_2) &&{}={} N_{1,2}\\
        S_{1,2}(&0, 1, x_2) &&{}={} 1\\
        S_{1,2}(&1, x_1, x_2) &&= 0
        \end{alignat*}

    \end{center}
    \end{minipage}
    \begin{minipage}{0.4\textwidth}

    \vspace{0pt}
        \begin{center}
        
        \begin{tikzpicture}
          \def\lvl{35pt}
        \node (c1) at (0, 0) [root] {$C_1$};
        \node (n10) at ($(c1) + (0bp, -30pt)$) [bddnode] {$N_{1,0}$};
        \node (n11) at ($(n10) + (0bp, -\lvl)$) [bddnode] {$N_{1,1}$};
        \node (n12) at ($(n11) + (0bp, -\lvl)$) [bddnode] {$N_{1,2}$};
        \node (true) at ($(n12) + (-15bp, -30pt)$) [bddterminal] {$\true$};
        \node (false) at ($(n12) + (15bp, -30pt)$) [bddterminal] {$\false$};
        
        \begin{scope}[on background layer]
          \draw [-stealth] (c1) -- (n10);
          \draw [lowedge] (n10) -- (n11);
          \draw [lowedge] (n11) -- (n12);
          \draw [lowedge] (n12) -- (true);
          \draw [highedge] (n12) -- (false);
          
          \path[-] (n11) edge [bend right, highedge] node {} (true);
          \path[-] (n10) edge [bend left, highedge] node {} (false);
        \end{scope}
        \end{tikzpicture}
        \end{center}
    \end{minipage}
\end{center}

Let $P(t)$ be the statement, $S_{i,t} = \ite{ (g(\vec{x}) < t) \lor (t > 2^b-2^{i-1} - 1)}
                                            { \lsb{i}{g(\vec{x})} }
                                            { N_{i,t} } $.

$P(1)$ holds as the initial node of all $\bddrt{C}{i}$ is always $N_{i,0}$.

Assume $P(t)$ for an arbitrary $t \leq 2^b - 1$.

Consider going to the next node. If we are at a terminal node, then the state stays the same, otherwise we go to the low or high edge based on $x_t$. Thus, to advance state, we replace $N_{i,t}$ with $\text{if \quad } x_t \text{\quad then \quad} \high{N_{i,t}}\text{\quad else \quad} \low{N_{i,t}}$.
\begin{align*}
    S_{i,t + 1} {}={} 
        &\text{if \quad  } (g(\vec{x}) < t) \lor (t > 2^b-2^{i-1} - 1) \\
        &\text{then \quad} \mhl{\lsb{i}{g(\vec{x})}} \\
        &\text{else \quad} \mhl{\itep{ x_t }{ \high{N_{i,t}} }{ \low{N_{i,t}} }}
\end{align*}
We substitute for $\high{N_{i,t}}$ and $\low{N_{i,t}}$.
\begin{align*}
    S_{i,t + 1} {}={} 
        &\text{if \quad } (g(\vec{x}) < t) \lor (t > 2^b-2^{i-1} - 1) \\
        &\text{then \quad} \lsb{i}{g(\vec{x})} \\
        &\text{else \quad} \itep { x_t }{\mhl{ \lsb{i}{t} }}{\mhl{ \itep{t < 2^b - 2^{i - 1} - 1}{N_{i,t+1}}{1} }}
\end{align*}
$x_t \land \lnot (g(\vec{x}) < t)$ implies $g(\vec{x}) = t$.
\begin{align*}
    S_{i,t + 1} {}={} 
        &\text{if \quad } (g(\vec{x}) < t) \lor (t > 2^b-2^{i-1} - 1)\\
        &\text{then \quad} \lsb{i}{g(\vec{x})} \\
        &\text{else \quad} \itep { x_t }{ \lsb{i}{\mhl{ g(\vec{x}) }} }{ \itep{t < 2^b - 2^{i - 1} - 1}{N_{i,t+1}}{1} }
\end{align*}

Two branches are now equivalent ($\lsb{i}{g(\vec{x})}$) and can be combined.
\begin{align*}
    S_{i,t + 1} {}={} 
        &\text{if \quad } (g(\vec{x}) < t) \lor (t > 2^b-2^{i-1}-1) \mhl{ \lor x_t }\\
        &\text{then \quad} \lsb{i}{g(\vec{x})} \\
        &\text{else \quad} \mhl{ \itep{t < 2^b - 2^{i - 1} - 1}{N_{i,t+1}}{1} }
\end{align*}
$g(\vec{x}) < t \lor x_t$ is equivalent to $g(\vec{x}) < t + 1$.
\begin{align*}
    S_{i,t + 1} {}={} 
        &\text{if \quad } \mhl{ (g(\vec{x}) < t + 1 }) \lor (t > 2^b-2^{i-1}-1) \\
        &\text{then \quad} \lsb{i}{g(\vec{x})} \\
        &\text{else \quad} \itep{t < 2^b - 2^{i - 1} - 1}{N_{i,t+1}}{1}
\end{align*}

In the outer else, $t \leq 2^b-2^{i-1} - 1$ (by the condition). In the inner else, $t$ is also $\geq 2^b-2^{i-1} - 1$.
\begin{align*}
    S_{i,t + 1} {}={} 
        &\text{if \quad } (g(\vec{x}) < t + 1) \lor (t > 2^b-2^{i-1} - 1) \\
        &\text{then \quad} \lsb{i}{g(\vec{x})} \\
        &\text{else \quad} \itep{t \mhl{=} 2^b - 2^{i - 1} - 1}{N_{i,t+1}}{1}
\end{align*}

In the inner else, $g(\vec{x}) > t$ by the outer condition and $t = 2^b - 2^{i-1} - 1$ by the inner condition; together, $g(\vec{x}) > 2^b - 2^{i-1} - 1$. For all such $g(\vec{x})$, $\lsb{i}{g(\vec{x})} = 1$ (as $2^b - 1 - 2^{i - 1})$ is the largest number at most $2^b - 1$ with the $i^\text{th}$ bit set to 0). Thus, we can merge two more branches.
\begin{align*}
    S_{i,t + 1} {}={} 
        &\text{if \quad } (g(\vec{x}) < t + 1) \lor (t \mhl{ + 1 } > 2^b-2^{i-1} - 1)\\
        &\text{then \quad} \lsb{i}{g(\vec{x})} \\
        &\text{else \quad} \mhl{N_{i,t+1}}
\end{align*}

Thus $P(t) \to P(t + 1)$ and thus $P(t)$ holds for all $1 \leq t \leq 2^b - 1$.

Consider $P(2^b - 1)$.
\begin{align*}
S_{i,2^b - 1} &= \ite{ (g(\vec{x}) < 2^b - 1) \lor (2^b - 1 > 2^b-2^{i-1} - 1)}
                                            { \lsb{i}{g(\vec{x})} }
                                            { N_{i,2^b} } \\                             
    S_{i,2^b - 1} &=  \lsb{i}{g(\vec{x})}
\end{align*}

Therefore $\ftobdd{g}$ and $C$ are equivalent.
\end{proof}

\textbf{Proposition 1 (first half).}

\begin{proof}
Let $(B, \vec{p}) = \operatorname{encode_{CATEG}}(\Pr(V))$, where $\Pr(V)$ is a distribution over the integers $\{0, .., 2^b-1\}$. 
It follows directly from Lemma \ref{sbkstructure} that $\sum_{i=1}^{b} 2^b-2^{i-1} = b2^b - 2^b + 1$ decision nodes are needed for $B$.

\end{proof}

\subsubsection{Bitwise Encoding of Integer Distributions}

Recall Algorithm \ref{alg:categbits}. By expanding the algorithm, we obtain the following encoding method for integers, which again is similar to how it is implemented in practice. 
\begin{align*}
    \bddrt{B}{1} &= x_\varepsilon\\
    \bddrt{B}{2} &= \text{if } x_\varepsilon \text{ then } x_1 \text { else } x_0\\
    \bddrt{B}{3} &= \text{if } x_\varepsilon \text{ then } (\text{if } x_1 \text{ then } x_{11} \text{ else } x_{10}) \text{ else } (\text{if } x_0 \text{ then } x_{01} \text{ else } x_{00})\\
    \vdots
\end{align*}

Our $x_s$ are again derived from the algorithm; here, we recurse on the two halves of the input vector, and our probability $p_s$ is the relative weight of the latter half of the vector (corresponding to the larger numbers, with a value of 1 for the MSB). While in the written algorithm we conditionally added a power-of-two, what we are in essence doing is probabilistically setting the value of the most significant bit. The subscript $s$ refers to the sequence of more significant bits that have already been determined; for example, $x_\varepsilon$ is deciding the value of the most significant bit (given an empty string, for no prior sequence), while $x_{11}$ is deciding the third MSB, given that the first two MSBs are $11$. 

We formalize this encoding in Definition \ref{bitwise_int_formal}. Let $s_i$ denote the $i^\text{th}$ bit of a 1-indexed bit string or number; for example, $0010_3 = 1$. Let $|s|$ denote the length of a bit string. Let $\concat$ denote concatenation for bits and bit strings.

Note that we use two different types of subscripts: decision variables are subscripted by bit sequences for identification purposes, while a subscript $i$ on values such as numbers and bitstrings represents the $i^\text{th}$ bit of the value. 

\begin{definition}
For a distribution $\Pr(V)$ over the integers $\{0,..,2^b-1\}$ with, define $\operatorname{encode_{\BWH}}(\Pr(V)) = (B, \vec{p})$ such that for each bit string $s$, $0 \leq |s| < b$, 
 there is a decision variable $x_s$ with probability $p_s = \Pr(V_{|s|+1} \mid  \bigwedge\limits_{i=1}^{|s|} V = s_i)$.

Intuitively, the bits are chosen left-to-right, and each decision variable chooses the next bit given a bit string of past choices. Consider the following examples, where the empty bit string is denoted as $\varepsilon$.
\begin{align*}
    \Pr(x_\varepsilon) &= \Pr(V_{1})\\
    \Pr(x_0) &= \Pr(V_{2} \mid \lnot V_{1})\\
    \Pr(x_1) &= \Pr(V_{2} \mid V_{1})\\
    \Pr(x_{0100}) &= \Pr(V_{5} \mid \lnot V_{1} \land V_{2} \land \lnot V_{3} \land \lnot V_{4})
\end{align*}
 We specify 
$
    \bddrt{B}{i} = x_{\bddrt{B}{1} \concat \bddrt{B}{2} \concat \hdots \concat \bddrt{B}{i - 1}}
$
for $1 \leq i \leq b$.
For example, $\bddrt{B}{1} = x_\varepsilon$, $\bddrt{B}{2} = x_{\bddrt{B}{1}}$, $\bddrt{B}{3} = x_{\bddrt{B}{1} \concat \bddrt{B}{2}}$, and so on.

\label{bitwise_int_formal}
\end{definition}

\begin{lemma} \label{bwhencodes}
For any distribution $\Pr(V)$ over the integers $\{0, .., 2^b-1\}$, $\operatorname{encode_{\BWH}}(\Pr(V))$ encodes $\Pr(V)$.
\end{lemma}

\begin{proof}
Let $(B, \vec{p}) = \operatorname{encode_{\BWH}}(\Pr(V))$. We prove that for any possible assignment to bits $\vec{a} \in \{0, 1\}^b$, $\Pr(B(\vec{X}) = \vec{a}) = \Pr(\vec{V} = \vec{a})$.
\begin{align*} 
    \Pr(B(\vec{X}) = \vec{a}) &= \prod_{i=1}^b \Pr(B_i(\vec{X})) = a_i \mid \bigwedge_{j=1}^{i-1} B_j(\vec{X}) = a_j) && \text{(Chain rule of probability)}\\
    &= \prod_{i=1}^b \Pr(\mhl{X_{B_1(\vec{X})\concat \dots \concat B_{i-1}(\vec{X})}} = a_i \mid \bigwedge_{j=1}^{i-1} B_j(\vec{X}) = a_j) && \text{(Bit specification)}\\
    &= \prod_{i=1}^b \Pr(X_{\mhl{a_1 \concat \dots \concat a_{i-1}}} = a_i \mid \bigwedge_{j=1}^{i-1} B_j(\vec{X}) = a_j) && \text{(Condition)}\\
    &= \prod_{i=1}^b \Pr(X_{a_1\concat \dots \concat a_{i-1}} = a_i \mid \bigwedge_{j=1}^{i-1} \mhl{X_{B(\vec{X})_1\concat \dots \concat B_{j-1}(\vec{X})}} = a_j) && \text{(Bit specification)}\\
    &= \prod_{i=1}^b \Pr(X_{a_1\concat \dots \concat a_{i-1}} = a_i) && \text{(Independence)}\\
    &= \prod_{i=1}^b \Pr(V_i = a_i \mid \bigwedge_{j=1}^{i-1} V_j=a_j) && \text{(As $\Pr(X_s) = p_s$)}\\
    &= \Pr(V = \vec{a}) && \text{(Chain rule of probability)}
\end{align*}
\end{proof}

\subsubsection{BDD Size of the BITWISE\_INT  encoding}

We prove the second half of Proposition 1.
\begin{propositionhalf}{1 (second half)} \label{bwhsize}
A discrete distribution over the integers $\{0, 1 \ldots, 2^b - 1\}$ compiles to a BDD of size $\Theta(2^b)$ when represented using \BWH ~(Algorithm \ref{alg:categbits}), with variables in flip evaluation order.
\end{propositionhalf}

\begin{lemma} \label{bwhstructure}
Let $\operatorname{encode_{\BWH}}(\Pr(V)) = (B, P)$, where $\Pr(V)$ is the distribution over the integers $\{0,...,2^b-1\}$.

$B$ will be exactly the BDD $C$ with the following structure:

For each root $1 \leq i \leq b$, for all bit strings $s$ of length less than $i$, let $C$ have a node $N_{i,s}$ corresponding to decision variable $x_s$.
\begin{align*}
    \low{N_{i, s}} &= \ite{ |s| = i - 1 }{ 0 }{N_{i, s\concat 0}}\\
    \high{N_{i, s}} &= \ite{|s| = i - 1} {1}{N_{i, s\concat 1}}
\end{align*}
The roots of the BDD $C_i$ each point to the corresponding node $N_{i,\varepsilon}$.

\end{lemma}

Intuitively, for each root corresponding to bit $i$ we consider up to a decision node of depth $i$ (corresponding to a prefix bitstring of length $i-1$); if we have not yet reached that depth, we instead go to a deeper decision node, with the appropriate new prefix.

As an example, the BDD for $b = 2$ follows (terminal nodes visually duplicated for clarity).
\vspace{-10pt}

\begin{center}
    \begin{minipage}[t]{0.15\textwidth}
    \vspace{0pt}
        \begin{center}
        \begin{tikzpicture}
          \def\lvl{30pt}
          \def\offset{15bp}
            \node (c1) at (0, 0) [root] {$C_1$};
            \node (n1e) at ($(c1) + (0bp, -30pt)$) [bddnode] {$N_{1,\varepsilon}$};
            \node (true) at ($(n1e) + (\offset, -\lvl-30pt)$) [bddterminal] {$\true$};
            \node (false) at ($(n1e) + (-\offset, -\lvl-30pt)$) [bddterminal] {$\false$};
            
            \begin{scope}[on background layer]
              \draw [-stealth] (c1) -- (n1e);
              \draw [lowedge] (n1e) -- (false);
              \draw [highedge] (n1e) -- (true);
        \end{scope}
        \end{tikzpicture}
        \end{center}
    \end{minipage}
    \begin{minipage}[t]{0.15\textwidth}
    \vspace{0pt}
    \begin{tikzpicture}
      \def\lvl{30pt}
      \def\offset{15bp}
        \node (c2) at (0, 0) [root] {$C_2$};
        \node (n2e) at ($(c2) + (0bp, -30pt)$) [bddnode] {$N_{2,\varepsilon}$};
        \node (n20) at ($(n2e) + (-\offset, -\lvl)$) [bddnode] {$N_{2,0}$};
        \node (n21) at ($(n2e) + (\offset, -\lvl)$) [bddnode] {$N_{2,1}$};
        \node (true) at ($(n21) + (0bp, -30pt)$) [bddterminal] {$\true$};
        \node (false) at ($(n20) + (0bp, -30pt)$) [bddterminal] {$\false$};
        
        \begin{scope}[on background layer]
          \draw [-stealth] (c2) -- (n2e);
          \draw [lowedge] (n2e) -- (n20);
          \draw [lowedge] (n20) -- (false);
          \draw [lowedge] (n21) -- (false);
          \draw [highedge] (n2e) -- (n21);
          \draw [highedge] (n20) -- (true);
          \draw [highedge] (n21) -- (true);
    \end{scope}
    \end{tikzpicture}
    \end{minipage}
\end{center}

\begin{proof}

Note that the BDD described has variable order following the evaluation order from Algorithm \ref{alg:categbits}; the direct descendant of a node corresponding to the decision variable $x_s$ will either correspond to the decision variable $x_{s \concat 1}$ or $x_{s \concat 0}$, both of which are later in the evaluation order. As this holds for all nodes, the variable order must also follow globally. 

We use a similar argument to that in the proof of Lemma \ref{sbkstructure}. 

We first argue that the BDD is reduced. The argument is along the same line: each root now points to a tree deciding the value of the bit. No node has isomorphic children, as the two children correspond to either (necessarily different) terminal nodes, or decision nodes corresponding to different variables. In addition, as before each tree is of a different depth, corresponding to a different final decision variable, and so no two nodes on the same decision variable can compute the same subfunction. Therefore, we cannot merge nodes.

It remains to show that the function described by the BDD $C$ is equivalent to that from the encoding $B$. 

Let $S_{i,j}(\vec{x})$ denote the state of $C_i$ at the level of the first decision variable whose bit string is of length $j$. See the following BDD and corresponding states as an example.
\begin{center}
    \begin{minipage}{0.3\textwidth}
    \begin{center}
        \begin{alignat*}{4}
        S_{2,0}(&x_\varepsilon,x_0,x_1) &&= N_{2,\varepsilon}\\
        S_{2,1}(&0, x_0, x_1) &&= N_{2,0}\\
        S_{2,1}(&1, x_0, x_1) &&= N_{2,1}\\
        S_{2,2}(&0, x_0, x_1) &&= x_0\\
        S_{2,2}(&1, x_0, x_1) &&= x_1
        \end{alignat*}
    \end{center}
    \end{minipage}
    \begin{minipage}{0.4\textwidth}
    \begin{center}
    \begin{tikzpicture}
      \def\lvl{30pt}
      \def\offset{15bp}
        \node (c2) at (0, 0) [root] {$C_2$};
        \node (n2e) at ($(c2) + (0bp, -30pt)$) [bddnode] {$N_{2,\varepsilon}$};
        \node (n20) at ($(n2e) + (-\offset, -\lvl)$) [bddnode] {$N_{2,0}$};
        \node (n21) at ($(n2e) + (\offset, -\lvl)$) [bddnode] {$N_{2,1}$};
        \node (true) at ($(n21) + (0bp, -30pt)$) [bddterminal] {$\true$};
        \node (false) at ($(n20) + (0bp, -30pt)$) [bddterminal] {$\false$};
        
        \begin{scope}[on background layer]
          \draw [-stealth] (c2) -- (n2e);
          \draw [lowedge] (n2e) -- (n20);
          \draw [lowedge] (n20) -- (false);
          \draw [lowedge] (n21) -- (false);
          \draw [highedge] (n2e) -- (n21);
          \draw [highedge] (n20) -- (true);
          \draw [highedge] (n21) -- (true);
    \end{scope}
    \end{tikzpicture}
    \end{center}
    \end{minipage}
\end{center}

Let $P(t)$ be the statement, $S_{i,t}(\vec{x}) =  \ite{ t \geq i }{ \bddrt{B}{i} }{ N_{i,\bddrt{B}{1} \concat \dots \concat \bddrt{B}{t}}}$.

$P(0)$ holds as there is no bit index greater than or equal to 0 and $S_{i,0}(\vec{x}) = N_{i,\varepsilon}$ by the placement of the roots.

Assume $P(t)$, which specifies $S_{i, t}$. Consider advancing to the next node in the BDD based on the assignment to decision variables (which does nothing if we are already at a terminal node):
\begin{align*}
    \operatorname{next(S_{i,t})} {}={} 
        &\text{if \quad } t \geq i\\
        &\text{then \quad}
            \bddrt{B}{i}\\
        &\text{else \quad} \itep
        { x_{\bddrt{B}{1} \concat \dots \concat \bddrt{B}{t}} } 
        { \high{N_{i,\bddrt{B}{1} \concat \dots \concat \bddrt{B}{t}}} }
        { \low{N_{i,\bddrt{B}{1} \concat \dots \concat \bddrt{B}{t}}} }
\end{align*}

We replace $x_{\bddrt{B}{1} \concat \dots \concat \bddrt{B}{t}}$ with $\bddrt{B}{t + 1}$.
\begin{align*}
    \operatorname{next(S_{i,t})} {}={} 
        &\text{if \quad } t \geq i\\
        &\text{then \quad}
            \bddrt{B}{i}\\
        &\text{else \quad} \itep{ \bddrt{B}{t + 1} }
            { \high{N_{i,\bddrt{B}{1} \concat \dots \concat \bddrt{B}{t}}} }
            { \low{N_{i,\bddrt{B}{1} \concat \dots \concat \bddrt{B}{t}}} }
\end{align*}

We replace the high and low edges by the definition:
\begin{align*}
    \operatorname{next(S_{i,t})} {}={} 
        &\text{if \quad } t \geq i &&\\
        &\text{then \quad} \bddrt{B}{i} &&\\
        &\text{else \quad} (&&\text{if \quad } \bddrt{B}{t + 1}\\
            & &&\text{then \quad}
                \left(
                \text{if \quad } t = i - 1 \text{\quad then \quad} 1 \text{\quad else \qquad} N_{i, \bddrt{B}{1} \concat \dots \concat \bddrt{B}{t}\concat 1}
                \right)\\
            & &&\text{else \quad}
                \left(
                \text{if \quad } t = i - 1 \text{\quad then \quad} 0 \text{\quad else \qquad} N_{i, \bddrt{B}{1} \concat \dots \concat \bddrt{B}{t}\concat 0}
                \right)
            )
\end{align*}

We rearrange the if conditions:
\begin{align*}
    \operatorname{next(S_{i,t})} {}={} 
        &\text{if \quad } t \geq i &&\\
        &\text{then \quad} \bddrt{B}{i} &&\\
        &\text{else \quad} (&&\text{if \quad } t=i-1\\
            & &&\text{then \quad}
                \left(
                \text{if \quad } \bddrt{B}{t+1} \text{\quad then \quad} 1 \text{\quad else \qquad} 0
                \right)\\
            & &&\text{else \quad}
                \left(
                \text{if \quad } \bddrt{B}{t+1} \text{\quad then \quad} N_{i, \bddrt{B}{1} \concat \dots \concat \bddrt{B}{t} \concat 1} \text{\quad else \qquad} N_{i, \bddrt{B}{1} \concat \dots \concat \bddrt{B}{t} \concat 0}
                \right)
            )
\end{align*}
\begin{align*}
    \operatorname{next(S_{i,t})} {}={} 
        &\text{if \quad } t \geq i &&\\
        &\text{then \quad} \bddrt{B}{i} &&\\
        &\text{else \quad} (&&\text{if \quad } t=i-1\quad\text{then \quad}
                \bddrt{B}{t+1}\quad\text{else \quad}
                N_{i, \bddrt{B}{1} \concat \dots \concat \bddrt{B}{t+1}}
            )
\end{align*}

The first two $\operatorname{then}$ branches collapse:
\begin{align*}
    \operatorname{next(S_{i,t})} {}={} 
        &\text{if \quad } t + 1 \geq i  
        \quad\text{then \quad} \bddrt{B}{i}  
        \quad\text{else \quad}
                N_{i, \bddrt{B}{1} \concat \dots \concat \bddrt{B}{t+1}}
\end{align*}

As the only node we now reach has bit string length $t+1$, $\operatorname{next}(S_{i,t}) = S_{i,t+1}$. Therefore $P(t) \to P(t+1)$.

Consider $P(b-1)$: $S_{i,b-1} =  \text{if \quad } b-1 \geq i \text{\quad then \quad} \bddrt{B}{i} \text{\quad else \quad} N_{i,\bddrt{B}{1} \concat \dots \concat \bddrt{B}{b-1}}$.

For all $i \in \{1, \dots, b-1\}$, we see that the root labeled $\bddrt{C}{i}$ reaches the value $\bddrt{B}{i}$. For $i = b$, the state reaches $N_{n, \bddrt{B}{1}\concat \dots \concat \bddrt{B}{b-1}}$, whose high and low edges are 1 and 0, respectively. The last step goes to $\itep{ x_{b_1\concat \dots \concat b_{b-1}} }{ 1 }{ 0 } = \bddrt{B}{b}$. Therefore $B$ and $C$ are equivalent.
\end{proof}

\textbf{Proposition 1 (second half).}

\begin{proof}

Let $\operatorname{encode_{\BWH}}(\Pr(V)) = (B, \vec{p})$, where $\Pr(V)$ is the distribution over the integers $\{0,...,2^b-1\}$. 
It directly follows from Lemma \ref{bwhstructure} that $B$ requires the following number of decision nodes.
\begin{gather*}
    \sum_{i=1}^b (\text{\# of bit strings of length less than $i$}) = 2^{b+1}-b-2
\end{gather*}
\end{proof}

\subsection{Encoding Uniform Integer Distributions}

\begin{propositionhalf}{2}\label{uniform}
$\forall n > 0$, $\Pr($UNIFORM$(n) = i) = \frac{1}{n}$ for $0 \leq i < n$. 
\end{propositionhalf}
\begin{proof}
We show this by strong induction on $n$.
When $n=1$ or $n=2$, this result trivially holds.
Consider $n=k$. Let $0 \leq i < k$, and $b = \lfloor \log_2(k) \rfloor$ as in the algorithm. 

\textit{Case 1:} $0 \leq i < 2^b$.
UNIFORM$(k) = i$ requires us to take the first branch of the if statement. The conditional addition is equivalent to setting the $b$ least significant bits of the number to $1$ with probability $\frac{1}{2}$ and $0$ with probability $\frac{1}{2}$. As each $0 \leq i < 2^b$ corresponds to a single unique binary string of length $b$, the probability of the bit sequence being equal to the wanted number is simply $(\frac{1}{2})^b$. Therefore, $\Pr($UNIFORM$(k) = i) = (\frac{1}{2})^b (\frac{2^b}{k}) =  \frac{1}{k}$.

\textit{Case 2:} $2^b \leq i < k$.
UNIFORM$(k) = i$ requires us to take the second branch of the if statement. Using our inductive hypothesis (as $k-2^b < k$), 
$\Pr($UNIFORM$(k) = i) = (\frac{k-2^b}{k})\Pr($UNIFORM$(k-2^b) = i-2^b) = (\frac{k-2^b}{k})(\frac{1}{k-2^b}) = \frac{1}{k}$, as wanted. 

\end{proof}

\section{Experiments}

\paragraph{Experimental Details}
All experiments were run on a server with a 2.20GHz CPU and 504GB RAM. WebPPL experiments were run on WebPPL v0.9.15; Psi experiments were run on Psi version \texttt{ec2cfc14a62a168afe7ce1d7269b92cf2882b830}. \verb#Dice.jl# experiments were run using the code available at \url{https://github.com/Juice-jl/Dice.jl/tree/arithmetic}. The implementations of each model run are available at the same repository.

\paragraph{Benchmark Model Descriptions}
We provide a small description of our benchmarks in this section with details of their sources. The code implementations of all models are available at the repository \url{https://github.com/Juice-jl/Dice.jl/tree/arithmetic}. 

\begin{enumerate}
    \item book: A model of flipping towards a target page in a book adapted from the Psi test directory~\citep{PsiDirectory}. 
    
    \item tugofwar: Adapted from a traditional tug-of-war example~\citep{huang2021aqua}, with values made discrete.
    \item caesar: The caesar-cipher example from Dice~\citep{HoltzenOOPSLA20}, with a different number of characters being observed. 

    \item ranking: A model for learning a ranking system, adapted from \cite{KisaLTPM14}.

\item radar1: A model of radar reception, adapted from a continuous model from Psi's~\citep{gehr2016psi} benchmark suite.

\item floydwarshall: An implementation of the Floyd-Warshall algorithm~\citep{10.1145/367766.368168} on a graph with edges of random weight. 

\item linear-extensions: A model counting linear extensions~\citep{https://doi.org/10.48550/arxiv.1802.06312} where we observe a partial order and get an output distribution over all matching total orders.

\item triangle: A model categorizing a triangle of random side lengths, adapted from the Psi test directory~\citep{PsiDirectory}.

\item gcd: A model checking if two random numbers are coprime implementing Euclid's algorithm~\citep{doi:10.1080/00029890.1938.11990797}.

\item disease: A discrete disease model taken from existing works~\citep{10.1007/978-3-030-44914-8_14}.

\item luhn: A probabilistic model of student IDs leveraging the Luhn algorithm~\citep{luhn}. 
\end{enumerate}

\paragraph{Additional Experimental Results}
We provide additional experimental results supplementing those in the main paper. 

BDD size serves as a proxy for how well knowledge compilation can exploit structure: the more compact the BDD, the smaller the representation of our function, and the faster weighted model counting can be executed. Table \ref{bddsizes} compares the resultant BDD size for various models when compiled in Dice.jl using a binary or one-hot encoding. We can see that in almost all cases the binary encoding results in a smaller BDD, in some cases much smaller. Note that these models are those for which the one-hot encoding did not timeout; it is likely for those models that timed out that the true compiled BDD size will end up being much larger than the binary encoded BDD size. 

\begin{table}[ht]
\centering
\caption{BDD sizes for probabilistic models using a binary vs one-hot encoding}\label{bddsizes}
\scriptsize
\begin{tabular}{lrr}\toprule
Benchmarks              &binary & one-hot  \\\midrule
tugofwar                &\textbf{2400}    &2821\\\cmidrule{1-3}
caesar-small            &\textbf{1304}    &3879 \\
caesar-medium           &\textbf{6344}             &17879  \\
caesar-large            &\textbf{12644}             & 35379 \\\cmidrule{1-3}
ranking-small           &\textbf{691 }   &1146  \\
ranking-medium          &\textbf{6218}     & 8297 \\
ranking-large           &\textbf{11491 }    &15680\\\cmidrule{1-3}
radar1                  &\textbf{181}     &332 \\\cmidrule{1-3}
floydwarshall-small     &\textbf{10}    &\textbf{10}  \\
floydwarshall-medium    &341    &\textbf{237} \\\cmidrule{1-3}
linear extensions-small &\textbf{29}     & 53\\
linear extensions-medium &\textbf{133}   & 257 \\
linear extensions-large &\textbf{997}    & 1538 \\\cmidrule{1-3}
triangle-small          &\textbf{10273 }   &150089 \\
triangle-medium         &\textbf{40419 }    & 1156785\\\cmidrule{1-3}
luhn-small              &\textbf{518 }   &1010 \\
luhn-medium             &\textbf{2899 }   &7361  \\
\bottomrule
\end{tabular}
\end{table}

\section{Beta Prior Application: Bayesian Network Parameter Learning}

One natural application of a Beta prior is as a prior distribution for learning the parameters of a (binary) Bayesian network. The task of Bayesian network parameter learning can be described as follows: given a set of data consisting of instantiations of network variables, we want to find the network parameters maximizing the probability of this data. One interesting case is when our data is incomplete; that is, there are some variables not given a value. In this setting, a Bayesian approach to parameter learning must consider all (exponentially many) possible instantiations of these missing values ~\citep{darwiche_2009}.

Using our Beta prior implementation described in the main paper, this setting can naturally be modeled within \texttt{Dice.jl}.. A Bayesian network can be expressed in the probabilistic program with Beta priors on each network parameter; a dataset can then be observed. By returning the distribution over our Beta parameters $\alpha$ and $\beta$, we obtain our posterior, a mixture over Beta distributions. These can then be manually combined to obtain the exact posterior density function.

We provide a brief example of this on the survey Bayesian network\footnote{https://www.bnlearn.com/bnrepository/}. The structure of this network is shown in Figure \ref{survey_example}; we simplify the network by making all variables binary so that the Beta is a suitable prior. We note that we can actually generalize to the non-binary case with a Dirichlet prior using an approach similar to that used for the Beta. We also provide an example missing dataset for the network; the the entry ? indicates a missing value. For this example, we focus on the specific parameter $\theta_{o|e} = \Pr(O=1|E=1)$, for which we give a uniform prior $\Beta(1, 1)$.

\begin{figure}[h]
\begin{minipage}{.5\textwidth}
\centering
    \begin{tikzpicture}[node distance=0.75cm, baseline=-1em]

      \node (A) at (0, 0) [bddnode] {\textsc{a}};
      \node (S) at ($(A) + (40bp, 0)$) [bddnode] {\textsc{s}};
      \node (E) at ($(A) + (20bp, -20bp)$) [bddnode] {\textsc{e}};
      \node (O) at ($(E) + (-20bp, -20bp)$) [bddnode] {\textsc{o}};
      \node (R) at ($(O) + (40bp, 0)$) [bddnode] {\textsc{r}};
      \node (T) at ($(R) + (-20bp, -20bp)$) [bddnode] {\textsc{t}};

      \draw[->] (A) -- (E);
      \draw[->] (S) -- (E);
      \draw[->] (E) -- (O);
      \draw[->] (E) -- (R);
      \draw[->] (O) -- (T);
      \draw[->] (R) -- (T);

    \end{tikzpicture}
\end{minipage}
\begin{minipage}{.5\textwidth}

    \begin{tabular}{c|c|c|c|c|c}
        A & S & E & O & R & T \\
        \hline
        1 & ? & ? & 1 & 1 & 1 \\
        1 & 0 & ? & 1 & 0 & 1 \\
        1 & 0 & 1 & ? & 0 & 1 \\
        1 & ? & 1 & 0 & ? & 1 \\
        0 & 0 & 1 & 1 & ? & 1 \\
        0 & 1 & ? & 1 & 1 & ? \\
        1 & ? & 0 & 0 & 1 & 0 \\
        0 & ? & 1 & ? & ? & ? \\
        1 & 1 & 1 & ? & 1 & ? \\
        1 & 0 & 0 & ? & 1 & 1 \\
        \hline
    \end{tabular}
\end{minipage}
\caption{Visualization of and example missing dataset for the survey Bayesian network.}
\label{survey_example}
\end{figure}

By running the program representing this task, we get a large output distribution over Beta parameters - a mixture over Beta distributions. Note that it does not contain as many entries as possible instantiations, as some complete variable instantiations result in the same posterior Beta. If we plot the corresponding mixture, we can get a posterior PDF --- note that this is an exact posterior recovered from the Beta parameters, in contrast to the approximate result one would get from sampling-based inference methods. The output and PDF visualization are given in Figure \ref{fig:posterior}.

\begin{figure}[tb]
    \begin{minipage}{.4\textwidth}
    \centering
       \begin{tabular}{c|c|c}
        $\alpha$ & $\beta$ & Pr(.)  \\
        \hline
        9 & 3 & 0.226\\
        8 & 4 & 0.207\\
        10 & 2 & 0.177 \\
        7 & 5 & 0.160 \\
        6 & 6 & 0.109 \\
        5 & 7 & 0.066 \\
        4 & 8 & 0.035 \\
        3 & 9 & 0.016 \\
        2 & 10 & 0.005 \\
        \hline
    \end{tabular}\\
    \end{minipage}
    \begin{minipage}{.6\textwidth}
        \includegraphics[scale=0.35]{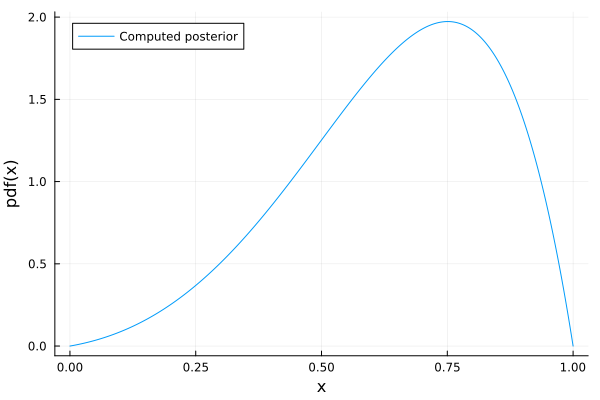}
    \end{minipage}

    \caption{Output posterior distribution and visualization. The output represents a mixture over Beta distributions, represented by their parameters $\alpha$ and $\beta$.}
    \label{fig:posterior}
\end{figure}

The code used for this example is available in the same \texttt{Dice.jl} repository (\url{https://github.com/Juice-jl/Dice.jl/tree/arithmetic}).

\end{document}